\documentclass[10pt,journal,compsoc]{IEEEtran}

\usepackage{graphicx}
\usepackage[utf8]{inputenc} % allow utf-8 input
\usepackage{url}            % simple URL typesetting
\usepackage{booktabs}       % professional-quality tables
\usepackage{amsfonts}       % blackboard math symbols
\usepackage{nicefrac}       % compact symbols for 1/2, etc.
\usepackage{microtype}      % microtypography
\usepackage[inline]{enumitem}

\usepackage[final]{changes}

\usepackage{algorithm} %format of the algorithm
\usepackage{algorithmic} %format of the algorithm
\usepackage{amsfonts}
\usepackage{lipsum}
\usepackage{graphicx}
\usepackage{subfigure}
\usepackage{amstext}
\usepackage{amsmath}
\usepackage{mathrsfs}
\usepackage{enumitem}
\usepackage{amsthm}
\usepackage{amsmath,amsfonts,amssymb,amsthm,epsfig,epstopdf,url,array}
\usepackage{multirow}
\usepackage{enumitem}
\usepackage{graphicx}
\usepackage[misc]{ifsym}
\usepackage[backend=bibtex,style=ieee,natbib=true,mincitenames=1]{biblatex} 
\setitemize{noitemsep,topsep=0pt,parsep=0pt,partopsep=0pt}

\newtheorem{theorem}{Theorem}
\newtheorem{lemma}{Lemma}

\theoremstyle{definition}
\newtheorem{definition}{Definition}

\newtheorem{assumption}{Assumption}
\newtheorem*{remark}{Remark}
\newtheorem{corollary}{Corollary}
\usepackage{xspace}
\mathchardef\mhyphen="2D
\newcommand{\kw}[1]{{\ensuremath {\mathsf{#1}}}\xspace}
\newcommand{\eat}[1]{}
\newcommand{\revision}[1]{{{\textcolor{black}{#1}}}} 
\makeatletter
\DeclareRobustCommand\onedot{\futurelet\@let@token\@onedot}
\def\@onedot{\ifx\@let@token.\else.\null\fi\xspace}
\def\eg{\emph{e.g}\onedot} 
\def\ie{\emph{i.e}\onedot} 
\def\etc{\emph{etc}\onedot} 
\def\wrt{w.r.t\onedot} 
\newcommand{\xhdr}[1]{\noindent{{\bf #1.}}}

\ifCLASSINFOpdf
\else
\fi
\begin{document}
%\title{A Restricted Black-box Adversarial Framework Towards Attacking Graph Embedding Models}
\title{Adversarial Attack Framework on Graph Embedding Models with Limited Knowledge}
% ~\IEEEmembership{Member,~IEEE,}
\author{Heng~Chang,
        Yu~Rong,
        Tingyang~Xu,
        Wenbing~Huang,
        Honglei~Zhang,
        Peng~Cui,
        Xin Wang,
        Wenwu~Zhu,~\IEEEmembership{Fellow,~IEEE},
        and~Junzhou~Huang,~\IEEEmembership{Fellow,~IEEE}% <-this % stops a space
\IEEEcompsocitemizethanks{

\IEEEcompsocthanksitem H. Chang is with the Tsinghua-Berkeley Shenzhen Institute, Tsinghua University, Beijing, China.
E-mail: changh17@mails.tsinghua.edu.cn
\IEEEcompsocthanksitem Y. Rong, T. Xu and H. Zhang are with the Tencent AI Lab, Tencent, Shenzhen, China.
E-mail: yu.rong@hotmail.com, Tingyangxu@tencent.com
% , zhanghonglei@gatech.edu
\IEEEcompsocthanksitem W. Huang is with Institute for AI Industry Research (AIR), Tsinghua University, Beijing, China. E-mail: hwenbing@126.com.
\IEEEcompsocthanksitem P. Cui, X. Wang and W. Zhu are with the Department of Computer Science and Technology, Tsinghua University.
E-mail: cuip@tsinghua.edu.cn, xin\_wang@tsinghua.edu.cn, wwzhu@tsinghua.edu.cn
\IEEEcompsocthanksitem J. Huang is with the Department of Computer Science and Engineering, University of Texas at Arlington.
E-mail: jzhuang@uta.edu
\IEEEcompsocthanksitem Yu Rong and Wenwu~Zhu are the corresponding authors. \protect \\
}% <-this % stops an unwanted space
% \thanks{Manuscript received April 19, 2005; revised August 26, 2015.}

}


\markboth{Journal of \LaTeX\ Class Files,~Vol.~14, No.~8, August~2020}%
{Shell \MakeLowercase{\textit{et al.}}: Bare Demo of IEEEtran.cls for Computer Society Journals}



\IEEEtitleabstractindextext{%
\begin{abstract}
With the success of the graph embedding model in both academic and industrial areas, the robustness of graph embeddings against adversarial attacks inevitably becomes a crucial problem in graph learning. Existing works usually perform the attack in a white-box fashion: they need to access the predictions/labels to construct their adversarial losses. However, the inaccessibility of predictions/labels makes the white-box attack impractical for a real graph learning system. 
This paper promotes current frameworks in a more general and flexible sense -- 
% we demand to attack various kinds of graph embedding models with black-box driven.
we consider the ability of various types of graph embedding models to remain resilient against black-box driven attacks.
We investigate the theoretical connection between graph signal processing and graph embedding models, and formulate the graph embedding model as a general graph signal process with a corresponding graph filter. Therefore, we design a generalized adversarial attack framework: \textit{GF-Attack}. Without accessing any labels and model predictions, \textit{GF-Attack} can perform the attack directly on the graph filter in a black-box fashion. We further prove that \textit{GF-Attack} can perform an effective attack without assumption on the number of layers/window-size of graph embedding models. To validate the generalization of \textit{GF-Attack}, we construct \textit{GF-Attack} on five popular graph embedding models. Extensive experiments validate the effectiveness of \textit{GF-Attack} on several benchmark datasets.

% With the success of the graph embedding model in both academic and industry areas, the robustness of graph embedding against adversarial attack inevitably becomes a crucial problem in graph learning. Existing works usually perform the attack in a white-box fashion: they need to access the predictions/labels to construct their adversarial losses. However, the inaccessibility of predictions/labels makes the white-box attack impractical for a real graph learning system. This paper promotes current frameworks in a more general and flexible sense -- we consider the ability of various types of graph embedding models to remain resilient against black-box driven attacks. We investigate the theoretical connection between graph signal processing and graph embedding models, and formulate the graph embedding model as a general graph signal process with a corresponding graph filter. Therefore, we design a generalized adversarial attack framework: GF-Attack. Without accessing any labels and model predictions, GF-Attack can perform the attack directly on the graph filter in a black-box fashion. We further prove that GF-Attack can perform an effective attack without assumption on the number of layers/window-size of graph embedding models. To validate the generalization of GF-Attack, we construct GF-Attack on five popular graph embedding models. Extensive experiments validate the effectiveness of GF-Attack on several benchmark datasets.
\end{abstract}

}

\maketitle

\IEEEdisplaynontitleabstractindextext

% \fi
\IEEEpeerreviewmaketitle

\IEEEraisesectionheading{\section{Introduction}\label{sec:introduction}}

% for IEEE Computer Society journal papers produced under \LaTeX\ using
% IEEEtran.cls version 1.8b and later.
% % You must have at least 2 lines in the paragraph with the drop letter
% % (should never be an issue)
% I wish you the best of success.~\cite{xu2019topology}

\IEEEPARstart{G}{raph} Embedding Models (GEMs)~\cite{cui2018survey,rong2020deep,huang2018adaptive,peng2020graph}, which elaborate the expressive power of deep learning on graph-structured data, have achieved remarkable success in various domains, such as drug discovery~\cite{duvenaud2015convolutional,li2018adaptive,rong2020selfsupervised},  social network analysis~\cite{ma2019detect, bian2020rumor,li2019semi}, computer version~\cite{zeng2019graph,wang2019graph}, medical imaging~\cite{wang2019graph,raju2020graph}, financial surveillance~\cite{paranjape2017motifs}, structural role classification~\cite{chang2021spectral,gu2020implicit} and automated machine learning~\cite{guan2021autogl}.
Given the increasing popularity and success of these methods, several recent papers have investigated the risk of GEMs against adversarial attacks, as other researchers had examined for convolutional neural networks~\cite{akhtar2018threat}. The papers~\cite{ICML2018Adversarial,KDD2018Adversarial,jin2021power} have already shown that various kinds of graph embedding methods, including GCN~\cite{ICLR2017SemiGCN}, DeepWalk~\cite{KDD2014Deepwalk}, \etc, are vulnerable to adversarial attacks. 

\begin{figure*}[htb]
\centering
\includegraphics [width=0.9\textwidth]{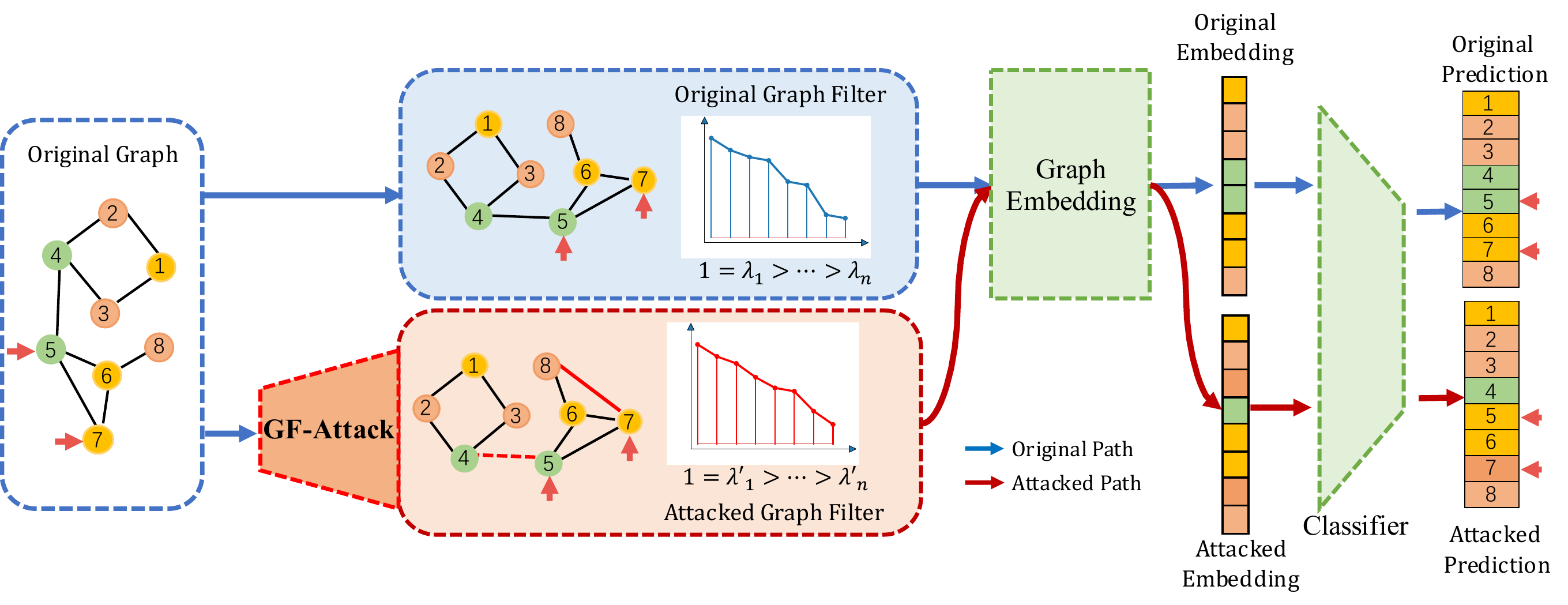}

\caption{The overview of the whole attack procedure of \textit{GF-Attack}. Given target vertices $5$ and $7$, \textit{GF-Attack} aims to misclassify them by attacking the graph filter and producing adversarial edges (edge $e_{45}$ deleted and edge $e_{78}$ added ) on the graph structure. The common graph embedding block refers to the general target GEM and can be any kind of the potential GEMs, illustrating the flexibility and extensibility of GF-Attack. In this vein, \textit{GF-Attack} would not change the target embedding model.}
\label{fig.attackoverview}
\end{figure*}

Undoubtedly, the potential attack risk is rising for modern graph learning systems. For instance, by sophisticated constructed social bots and following connections, it's possible to fool the recommendation system equipped with GEMs to give wrong recommendations. Another example is from the credit prediction model. The model tends to suppose that users connecting with high-credit users also have high credits. By constructing fake connections with high-credit users, fraudsters can easily fool the credit prediction model and lead to severe consequences. \revision{This potential risk calls for the attention on strengthening the security of GEMs. In this vein, the need for a new adversarial attack framework is especially essential for a better understanding of the adversarial examples existing in graphs as well as the design of more robust GEMs.}

Regarding the amount of information from both the target model and data required for the generation of adversarial examples, all graph adversarial attackers fall into three categories (arranged in ascending order of difficulties):
%may be a footnote here. 
\begin{itemize}
    \item White-box Attack (\textbf{WBA}): the attacker can access any information, namely, input data (e.g., adjacency matrix and feature matrix), labels, gradients, model parameters, model predictions, etc. However, this situation could be impractical since such information usually is well protected or inaccessible in the real world.
    \item Practical White-box Attack (\textbf{PWA}) (or Grey-box Attack): the attacker can access any information except the model gradients and parameters. Still, such information of GEMs is also difficult for attackers to obtain. For example, users in the credit prediction model are usually encoded to be anonymous and the labels of users are hard to be reached.
    \item Restricted Black-box Attack (\textbf{RBA}): the attacker can only access the adjacency matrix and attribute matrix. Access to parameters, labels, and predictions is prohibited. Being the most difficult but most practical setting, RBA is more natural in a real-world scenario, because the input data is always the only information we can easily obtain in most situations.
\end{itemize}

\begin{table}[!t]
\renewcommand{\arraystretch}{1.11}
\caption{Summary of information accessibility under three attack settings.}  \label{tab.settings}%
\vspace{-7mm}
\begin{center}
\resizebox{\columnwidth}{!}{%
\begin{tabular}{|l|l|l|l|l|}
\hline
Settings & Parameters & Predictions & Labels & Input Data  \\
\hline
WBA & \checkmark & \checkmark & \checkmark     & \checkmark \\
\hline
PWA &  & \checkmark & \checkmark  & \checkmark \\
\hline
RBA &  &  &  & \checkmark \\
\hline
\end{tabular}
}
\end{center}
\vspace{-7mm}
\end{table}

Table~\ref{tab.settings} summarizes the information accessibility under different adversarial attack settings.
Despite the fruitful results \cite{sun2018adversarial,KDD2018Adversarial,ICLR2019Meta} which absorb ingredients from exiting adversarial methods on convolutional neural networks, obtained in attacking graph embeddings under both WBA and PWA setting, however, the target model parameters/gradients, the labels, and predictions are seldom accessible in real-life applications. In other words, it is almost impossible for the WBA and PWA attackers to perform a threatening attack on real systems.
Meanwhile, current RBA attackers are either reinforcement learning-based \cite{ICML2018Adversarial}, which has low computational efficiency or derived merely only from the structure information without considering the feature information~\cite{icml2019adversarial}. Therefore, how to perform the effective adversarial attack toward GEM relying on the input adjacency matrix and attribute matrix, a.k.a., RBA setting, is still more challenging yet meaningful in practice.

The core task of the adversarial attack on the GEM is to damage the quality of output embeddings to harm the performance of downstream tasks within the manipulated features or graph structure, i.e., vertex or edge insertion/deletion. Namely, finding the embedding quality measure to evaluate the damage on graph embeddings is vital. For the WBA and PWA attackers, they have enough information to construct this quality measure, such as the loss function of the target model. In this vein, the attack can be performed by simply maximizing the loss function reversely given the known labels, either through gradient ascent \cite{ICML2018Adversarial} or a surrogate model \cite{KDD2018Adversarial,ICLR2019Meta}. However, the RBA attacker cannot employ the limited information to recover the loss function of the target model. In a nutshell, the biggest challenge of the RBA attacker is: how to figure out the goal of the target model barely by the input data. 

In this paper,  we try to understand GEMs from a new perspective and propose an attack framework: \textit{GF-Attack}, which can perform an adversarial attack on various kinds of GEMs. Specifically, we formulate a GEM as a general graph signal processing with a corresponding graph filter which can be computed by the input adjacency matrix. %Table~\ref{tab:allres} provides some representative GEMs with formulations from the perspective of general graph signal processing as examples. 
Therefore, we employ the graph filter as well as the corresponding feature matrix to construct the embedding quality measure as a $T$-rank approximation problem. In this vein, instead of attacking the loss function, we aim to directly attack the graph filter of given GEMs without knowing the labels and predictions. Therefore, \textit{GF-Attack} can perform an attack in a restricted black-box fashion by only assuming what type is the victim model. Furthermore, 
through evaluating this $T$-rank approximation problem, \textit{GF-Attack} is capable of performing the adversarial attack on any GEM that can be formulated as a general graph signal processing. Figure~\ref{fig.attackoverview} provides an overview of the whole attack procedure of \textit{GF-Attack}. 
\revision{Moreover, by theoretically analyzing the alternate adversarial loss on graph filter, we show that when we construct the attack loss in \textit{GF-Attack} with higher-order polynomial, the generated adversarial edges could perform more effective attacks on GEMs with a smaller number of layers/window-size.}
%targeting on models with larger layers/window-size could also conduct effective attack on models with smallest orders. It indicates that \textit{GF-Attack} can 
To demonstrate the effectiveness of \textit{GF-Attack} attacking various kinds of GEMs, we give the quality measure construction for four popular GEMs (GCN, SGC, DeepWalk, LINE). Empirical results show that our general attack method is capable of effectively performing adversarial attacks on popular unsupervised/semi-supervised GEMs on real-world datasets in a restricted black-box fashion.

The primary version has been published in the Thirty-Fourth AAAI Conference on Artificial Intelligence (AAAI-20)~\cite{chang2020restricted}. The contributions of the conference version are summarized as follows:
\begin{itemize}
    \item We construct the theoretical connection between GEM and graph signal processing with the corresponding graph filters. 
    \item We formulate the embedding quality measure as a $T$-rank approximation problem via graph filters, and the RBA setting is satisfied in this way. A general attack framework \textit{GF-Attack} is proposed accordingly.
    \item Experiments towards attacking four popular GEMs on real-world datasets reveal the effectiveness of the proposed framework \textit{GF-Attack}.
\end{itemize}

In the conference version~\cite{chang2020restricted}, \textit{GF-Attack} performs attack with the additional assumption on the number of layers/window-size in GEMs. In this extended version, we further analyze the generalization ability of \textit{GF-Attack} on attacking GEMs with different layers/window-size to further remove the dependency on this assumption, especially from the perspective of theoretical findings. We list the key additional contributions here, independently:
\revision{
\begin{itemize}
    \item By investigating the adversarial loss of \textit{GF-Attack}, we prove that \textit{GF-Attack} can perform the effective attack without additional assumption on the number of layers/window-size of GCNs and sampling-based GEMs,
    which is an important step further to a more ideal black-box attack setting.
    % Supplementary experiments are conducted to show the consistency with our theoretical analysis.
    %We conduct the theoretical analysis about the relaxing of the adversarial loss for both graph convolutional networks and sampling-based graph embedding.
    %Theoretical analysis on the adversarial loss for both graph convolutional networks and sampling-based graph embedding is provided to avoid utilizing the information from the number of layers/window-size.
    \item A parameterized-filter variant of GCNs, ChebyNet, is included as victim model in experiments to further demonstrate the attack ability of \textit{GF-Attack} aside from fixed-filter GCNs.
    \item We adopt a more black-box setting, \ie using the same attack loss for all victim models, to further empirically validate the effectiveness of \textit{GF-Attack} under both poisoning and evasion settings.
    \item Ablation studies on more benchmarks focusing on computational efficiency and multi-edge attack complete the empirical results. The additional results further demonstrate that \textit{GF-Attack} enjoys both effectiveness and efficiency on all benchmarks.
\end{itemize}
}
\eat{
The paper is organized as follows. Section~\ref{sec.related} discusses the related work and Section~\ref{sec.Background} provides with notations throughout the paper and defines the problem. We describe the methodology and design of \textit{GF-Attack} in Section~\ref{sec.GSPG} and report the experimental results against the state-of-the-art baselines in Section~\ref{sec.exp}. Afterwards, Section~\ref{sec.conclusion} concludes the paper.
}

\section{Related work}\label{sec.related}
\xhdr{Graph Learning and Graph Embedding Models}
Graph embedding models (GEMs)~\cite{goyal2018graph,cui2018survey,zhang2020deep} are essential techniques for graph analytic tasks. A taxonomy of GEMs can be broadly divided into four kinds~\cite{goyal2018graph}: (i) factorization methods, (ii) random walk (sampling-based) techniques, (iii) deep learning, and (iv) other miscellaneous strategies. Among them, random walk and deep learning-based methods are the most representative categories. For random walk techniques, DeepWalk~\cite{KDD2014Deepwalk}, LINE~\cite{WWW2015Line} and node2vec~\cite{grover2016node2vec} adopt SkipGram, a neural language model that aims to maximize the co-occurrence probability among the words that appear within a window, for graph embeddings. These methods then preserve different orders of network proximity with the learned low-dimensional vectors. We denote the methods from this category as sampling-based GEMs. As for deep learning, Graph Convolutional Networks (GCNs) such as GCN~\cite{ICLR2017SemiGCN} and SGC~\cite{sgc_icml19} generalize the deep neural model to non-Euclidean domains and learn the low-dimensional graph embeddings to maintain different scales of structural similarity. 
\revision{Regarding to whether the graph filters in GCNs parameterized, we can category GCNs as fixed-filter (\eg, GCN and SGC), and parameterized-filter (\eg, ChebyNet~\cite{Defferrard2016ChebNet}, GAT~\cite{velickovic2018gat} and GraphHeat~\cite{xu2019graph}) variants. In this work, our theoretical analysis focuses on fixed-filter GCNs, and the empirical experiments are evaluated by viewing both types of variants as victim models.}
For an explanation of GEMs, \citet{WSDM2018NetworkEmbedding} shows some insights on the understanding of sampling-based GEMs. However, they focus on proposing new graph embedding methods rather than building up a theoretical connection.

\xhdr{Adversarial Attacks on Graphs}
Recently, adversarial attacks on deep learning for graphs have drawn unprecedented attention from researchers.
\citet{ICML2018Adversarial} exploits a reinforcement learning-based framework under the RBA setting. However, they restrict their attacks on edge deletions only for vertex classification. Even more, they do not evaluate the \textit{transferability}~\cite{tramer2017space}, which denotes the phenomenon that the adversarial examples generated for a specific model can also be harmful when they are used on another model. Transferability is an important ability of adversarial examples.
\citet{KDD2018Adversarial} proposes attacks based on a surrogate model and they can do both edge insertion/deletion in contrast to \citet{ICML2018Adversarial}. But their method utilizes additional information from labels, which is under the PWA setting. Further, \citet{ICLR2019Meta} utilizes meta-gradients to conduct attacks under black-box setting by assuming the attacker uses a surrogate model same as \citet{KDD2018Adversarial}. Their performance highly depends on the assumption of the surrogate model, and also requires label information. Moreover, they focus on the global attack setting.
\citet{xu2019topology} proposes a gradient-based method under the WBA setting and overcomes the difficulty brought by discrete graph data.
In the meantime, \citet{wu2019adversarial} also suggests using the integrated gradients to search for edges and features as adversarial examples under the WBA setting.

\citet{icml2019adversarial} considers a different adversarial attack task on vertex embeddings under the RBA setting. Inspired by \citet{WSDM2018NetworkEmbedding}, they maximize the loss obtained by DeepWalk with matrix perturbation theory while only considering the information from the adjacency matrix.
%In contrast, we focus on semi-supervised learning on node classification combined with features.
Besides, several other works also open doors for interesting research directions in many ways. \citet{li2020adversarial} proposes an iterative learning framework to hide targeted individuals from the community detection task by GEMs in a black-box fashion. \citet{entezari2020all} finds that only the high-rank singular components of the graph are affected by the attack method Nettack~\cite{KDD2018Adversarial}.
Then \citet{entezari2020all} suggests that the power of Nettack can be greatly reduced if a low-rank approximation of the graph is utilized in contrast to the original clean graph. This finding is consistent with our analysis in measuring the embedding quality from Section~\ref{sec.GSPG} that we can optimize the low-rank approximation of the output embeddings reversely to generate adversarial edges.
Meanwhile, 
% \citet{sun2020non} further expands the adversarial candidates for structural attack from edges to nodes. \citet{sun2020non} develops a reinforcement learning-based method to modify the adversarial information of the adversarial nodes sequentially and then injects the nodes into the clean graph for poisoning. 
\citet{ma2020towards} studies the problem of the black-box attacks on graph neural networks by enforcing a novel constraint. In \citet{ma2020towards}, attackers can only have access to a subset of vertices. Meanwhile, only a small number of candidates can be selected as target vertices. At the same time, \citet{vidanage2020graph,zhang2020adversarial} consider the adversarial attack on graph neural networks from a new perspective. They focus on perturbing the graph structure to degrade the quality of the task of deep graph matching. Some efforts~\cite{elinas2020variational,zhang2020gnnguard,wu2020graph,chang2021not} have also been paid on the defense against the adversarial attack on GEMs recently.

Remarkably, despite all the above-introduced works except \citet{ICML2018Adversarial} showing the existence of transferability in GEMs by experiments, they all lack theoretical analysis on this implicit connection. In this work, for the first time, we theoretically connect different kinds of GEMs and propose a general optimization problem from parametric graph signal processing. An effective algorithm is developed afterwards under the RBA setting.

\section{Preliminaries}\label{sec.Background}
Let $G(\mathcal{V},\mathcal{E})$ be an attributed graph, where $\mathcal{V}$ is a vertex set with size $n = |\mathcal{V}|$ and $\mathcal{E}$ is an edge set with $|\mathcal{E}|$ edges.  Denote $A \in \{0,1\}^{n \times n}$ as an adjacency matrix and $X \in \mathbb{R}^{n \times l}$ as a feature matrix with dimension $l$\deleted[id==RR]{ for vertices}. $D_{ii} = \sum_{j}A_{ij}$ refers to the degree matrix. $\text{vol}(G) = \sum_{i}\sum_{j}A_{ij} = \sum_{i}D_{ii}$ denotes the volume of $G$. For consistency, we denote the perturbed adjacency matrix as $A'$ and the normalized adjacency matrix as $\hat{A} = D^{-\frac{1}{2}}AD^{-\frac{1}{2}}$. Symmetric normalized Laplacian and random walk normalized Laplacian are referred as $L^{sym} = I_n - D^{-\frac{1}{2}}AD^{-\frac{1}{2}}$ and $L^{rw} = I_n - D^{-1}A$, respectively. \revision{We also denote the attributed graph after attack as $G'(\mathcal{V'},\mathcal{E'})$, and the corresponding adjacency matrix as $A'$. The other notations of the perturbed graph are defined analogously.}

To cope with the data with graph structure in ML tasks, GEMs aim to encode sufficient features in graphs. Concretely, given a graph $G$, the goal is to learn a mapping function $\mathscr{M}: (A,X) \rightarrow \mathbb{R}^{n \times d}$  on the graph that represent vertex into a $d$-dimensional vector space with the preservation of structural ($A$) and non-structural ($X$) properties. According to the demand of random walk paths (RWs), deep learning-based GEMs generally fall into two categories \cite{cai2018comprehensive}: convolution-based Graph Neural Networks (GCNs), \emph{e.g.} GCN~\cite{ICLR2017SemiGCN}, and sampling-based GEMs, \emph{e.g.} DeepWalk~\cite{perozzi2014deepwalk}.

Given a GEM $\mathscr{M}_\Theta$ parameterized by $\Theta$ and a graph $G(\mathcal{V}, \mathcal{E})$,  the adversarial attack on graph aims to perturb the learned vertex representation $Z = \mathscr{M}_{\Theta}(A, X)$ to damage the performance of the downstream learning tasks. In a summary, three components in graphs can be attacked as targets:
\begin{itemize}
  \item Attack on $\mathcal{V}$: Add/delete vertices in graphs. This operation may change the dimension of the adjacency matrix $A$.
  \item Attack on $\mathcal{E}$: Add/delete edges in graphs. This operation would lead to the changes of entries in the adjacency matrix $A$. This kind of attack is also known as \emph{structural attack}.
  \item Attack on $X$: Modify the attributes attached on vertices.
\end{itemize}
In this paper, we mainly focus on studying the adversarial attacks on the graph structure, i.e., adding/deleting the edges in graphs, since attacking $\mathcal{E}$ is more practical than others in real applications \cite{CIKM2012Gelling}.

\revision{Meanwhile, considering in which stage the adversarial attack happens, we can also category the attack that happens at the test time as \emph{evasion} attack, and at the training time as \emph{poisoning} attack~\cite{KDD2018Adversarial}. In this work, we mainly focus on evasion attack, since it is more realistic in comparison to the accessibility to training data.}

\subsection{Graph Signal Filtering}
Graph Signal Processing (GSP) extends the concepts in Discrete Signal Processing and focuses on the analysis and processing of the data points whose relations are modeled as graphs~\cite{shuman2013GSP,ortega2018graph}.  Similar to DSP, these data points can be treated as \emph{signals}. Thus the definition of \emph{graph signal} is:
\begin{definition}[graph signal]\label{def.gs}
Given a graph $G(\mathcal{V}, \mathcal{E})$, a \emph{graph signal} $\mathbf{x}$ is a mapping from vertex set $\mathcal{V}$ to real numbers:
\begin{align}
    \notag \mathbf{x}: & \mathcal{V} \rightarrow \mathbb{R},\\
               & v_i \mapsto x_i.
\end{align}
\end{definition}
In Definition~\ref{def.gs}, each signal $\mathbf{x}$ is isomorphic in $G$. We can rewrite it into a vector: $\mathbf{v}=[v_1,\ldots,v_n]$. In this sense, the feature matrix $X$ can be treated as graph signals with $l$ channels.

To understand the graph signal $x$, it's essential to consider the graph structure. In general, a graph filter $\mathscr{H}$ is a system that takes a graph signal $\mathbf{x}$ as input and produces a new signal as an output. Namely, $\mathscr{H}$ performs a signal transformation on the original graph signals. 
In traditional DSP, \emph{shift filter ($z$-transform)} is a basic but non-trivial transformation which delays the signals in the time domain. Thus we can extend the definition of \emph{shift filter} to graph signals:
\begin{definition}[graph-shift filter]
Given a graph $G(\mathcal{V}, \mathcal{E})$, a graph-shift filter $S \in \mathbb{R}^{n\times n}$ is a matrix satisfying: $\forall i \neq j$ and $e_{ij} \notin \mathcal{E}$ , $S_{ij} = 0$. 
\end{definition}
The graph-shift filter $S$ reflects the locality property of graphs, i.e., it represents a linear transformation of the signals on one vertex and its neighbors. It's the basic building blocks to construct $\mathscr{H}$. Some common choices of $S$ include the adjacency matrix $A$ and the Laplacian $L=D - A$, where $D$ is the degree matrix $D_{ii} = \sum_{j=1}^{n}A_{ij}$. 

%According to the definition of graph-shift filter, we conclude that:

%\begin{lemma}\label{lemma:GSF}
%\cite{sandryhaila2013discrete} Given a graph-shift filter $S$, $\mathscr{H}$ is \emph{linear}, \emph{shift-invariant} iff $\forall \mathbf{x}_1, \mathbf{x}_2 \in \mathbb{R}^N$,$a,b \in \mathbb{R}$:
%\begin{align}
%    \notag\mathscr{H}(a\mathbf{x}_1 + b\mathbf{x}_2) &= a\mathscr{H}(\mathbf{x}_1) + b\mathscr{H}(\mathbf{x}_2)\\
%    \notag S\mathscr{H}(\mathbf{x}_1)&=\mathscr{H}(S\mathbf{x}_1)
%\end{align}
%\end{lemma}

%\subsection{GEMs}

\subsection{Adversarial Attack Definition}
Formally, given a fixed budget $\beta$ indicating that the attacker is only allowed to modify $2\beta$ entries in $A$ (undirected), the adversarial attack on a GEM \revision{$\mathscr{M}$} can be formulated as \cite{icml2019adversarial}:
% \begin{align}\label{equ.problemdef}
%     \arg\max\limits_{A'}  & \,\, \mathcal{L}_{\textrm{attack}}(A', X)\\ 
%     \notag\text{s.t.}\; &\notag Z = \mathscr{M}_{\Theta}(A', X), \\ \Theta^{*} &= \arg\min_{\Theta}\mathcal{L}_{\mathscr{M}}(\Theta; A', X),
%                         \notag\| A' - A\|_{0} = 2\beta,
% \end{align}
\begin{align}\label{equ.problemdef-ori}
    \revision{\notag
    \arg\max\limits_{A'}}  & \,\, \revision{\mathcal{L}_{\textrm{attack}}(A', X; \Theta) = \mathcal{L}_{\textrm{attack}}(Z)}\\ 
    \revision{ \notag
    \text{s.t.}\; } & \revision{\notag Z = \mathscr{M}(A', X; \Theta^{*}),}\\ 
    & \revision{ \notag
    \Theta^{*} = \arg\min_{\Theta}\mathcal{L}_{\textrm{model}}(A', X; \Theta), } \\
    & \revision{
    \| A' - A\|_{0} = 2\beta,}
\end{align}
\revision{where $Z = \mathscr{M}(A', X; \Theta^{*})$ is the embedding output of the model $\mathscr{M}$ with the optimal model parameters $\Theta^{*}$. $\mathcal{L}_{\textrm{model}}(A', X; \Theta)$ is the loss function of the victim model minimized by $\Theta$. $\mathcal{L}_{\textrm{attack}}(Z)$ is defined as the attack loss function measuring the damage on output embeddings. 
% Lower loss in $\mathcal{L}_{\textrm{attack}}(A', X)$ corresponds to higher quality. 
For the WBA setting, $\mathcal{L}_{\textrm{attack}}(Z)$ can be defined as the minimization of the target loss, i.e.,  $\mathcal{L}_{\textrm{attack}}(Z) = \inf\limits_{\Theta}\mathcal{L}_{\textrm{model}}(A', X; \Theta)$. This is generally a bi-level optimization problem since we need to re-train the model during attack to keep $\Theta^{*}$ as optimal in $Z$. In this work, we consider the evasion attack scenario, where $\Theta^{*} = \arg\min_{\Theta}\mathcal{L}_{\textrm{model}}(A, X; \Theta)$ are learned on the clean graph and remains unchanged during attack. In this way, we can treat the model parameters $\Theta^{*}$ as constants, which eases the construction of the attack loss from $\mathcal{L}_{\textrm{attack}}(Z)$ to $\mathcal{L}_{\textrm{attack}}(A', X)$. }
% In this way, the optimization problem of model will be collapsed to:
% \begin{align}\label{equ.problemdef}
%     \revision{\notag
%     \arg\max\limits_{A'}}  & \,\, \revision{\mathcal{L}_{\textrm{attack}}(\mathscr{M}(A', X; \Theta^{*}))}\\ 
%     \revision{\notag
%     \text{s.t.}\; } & \revision{ \Theta^{*}  = \arg\min_{\Theta}\mathcal{L}_{\textrm{model}}(A, X; \Theta) = \textrm{constant}, }\\
%     & \revision{ 
%     \| A' - A\|_{0} = 2\beta.}
% \end{align}

\revision{Theoretically analyzing poisoning attacks is usually harder since the subsequent learning of $\Theta^{*} = \arg\min_{\Theta}\mathcal{L}_{\textrm{model}}(A, X; \Theta)$ should be considered~\cite{KDD2018Adversarial}, therefore we choose to concentrate on evasion setting and leave the analysis under poisoning setting as future work. Note that though our loss is designed under the evasion setting, our main experimental results are under both settings, which demonstrate that our proposed attack loss can effectively destroy the performance of GEMs in practice.
}

\section{Methodologies}\label{sec.GSPG}
%From the GSP perspective, we can treat the feature matrix $X$ as graph signals with $l$ channels.
\revision{From the perspective of GSP,  we can formulate the process of generating embeddings $Z = \mathscr{M}(A, X; \Theta^{*})$ as a generalization of signal processing, according to the graph filtering together with feature transformation:} 
\begin{align}
    \revision{\notag\kw{graph}\text{ }\kw{filtering:} } & \revision{\tilde{X} = \mathscr{H}(X) = h(S)X}, \\
    \revision{\kw{feature}\text{ }\kw{transformation:} } & \revision{Z =  \sigma(\tilde{X}\Theta^{*})},
    \label{equ.GF-Attack}
\end{align}
\revision{
where $\sigma(\cdot)$ denotes the activation function, and $\Theta \in \mathbb{R}^{l \times l'}$ denotes the transformation weights from $l$ input channels to $l'$ output channels. 
$\mathscr{H} = h(S)$ denotes a graph signal filter, where $S = f(A)$ is the graph-shift filter and a function of adjacency matrix $A$, where the function is decided by a specific GEM.
$\mathscr{H}$ is usually constructed by a polynomial function $h(x)=\sum_{i=0}^La_ix^i \in \mathbb{R}^{n \times n}$ with graph-shift filter $S$. Many GEMs, including GCN, Deepwalk, etc, can be formulated as Eq.~\eqref{equ.GF-Attack} with different graph signal filter $\mathscr{H}$. Table~\ref{tab:allres} summarizes the graph filter of different GEMs. We can find that the formulation from the process in Eq.~\eqref{equ.GF-Attack} is so general that we can have the following assumption on the victim model:
}

\begin{assumption}\label{ass.GSP Assumption}
\revision{
For a given victim GEM $\mathscr{M}$, the output embedding of $\mathscr{M}$ is learned through the process of the generalization of GSP as analyzed in ~\eqref{equ.GF-Attack}.
}
\end{assumption}

\revision{
Under Assumption~\ref{ass.GSP Assumption}, since the model parameters $\Theta^{*}$ are kept as constant as discussed before, it's intuitively adequate to focus on attacking the process of graph filtering $\tilde{X} = \mathscr{H}(X) = h(S)X$ for most GEMs. As a result, we can directly damage the quality of the output embedding $Z$ through attacking $\tilde{X}$ by destroying the graph signal filter $\mathscr{H} = h(S)$. In this way, the optimization problem under our setting will be collapsed to:
}
% \begin{align}\label{equ.problemdef-simp}
%      \revision{ \arg\max\limits_{A'} \,\, \mathcal{L}_{\textrm{attack}}(A', X), \text{s.t.}\; S' = f(A'), \| A' - A\|_{0} = 2\beta.}
% \end{align}
\begin{align}\label{equ.problemdef-simp}
    \revision{\notag
    \arg\max\limits_{A'}}  & \,\, \revision{\mathcal{L}_{\textrm{attack}}(A', X) = \mathcal{L}_{\textrm{attack}}(\tilde{X'})}\\ 
    \revision{
    \text{s.t.}\; } & \revision{\tilde{X'} = h(S')X, S' = f(A'), \| A' - A\|_{0} = 2\beta.}
\end{align}

\revision{
We name this way of constructing attack loss $\mathcal{L}_{\textrm{attack}}$ targeting the graph signal filter in the victim GEM under Assumption~\ref{ass.GSP Assumption} as a general framework, \emph{Graph Filter Attack~(\textit{GF-Attack})}. Since the attack loss $\mathcal{L}_{\textrm{attack}}$ in \textit{GF-Attack} does not involve the model parameters $\Theta$ and predictions, \textit{GF-Attack} is a RBA framework for generating adversarial examples as discussed in the Introduction. 
}

%We call this general model \emph{Graph Filter Attack~(GF-Attack)}. \textit{GF-Attack} introduces the trainable weight matrix $\Theta$ to enable stronger expressiveness which can fuse the structural and non-structural information. 

\begin{table*}[!t]
  \centering
%   \scriptsize
%   \setlength{\tabcolsep}{1.6mm}
    \vspace{-1.5mm}
    \caption{From the perspective of general graph signal processing, we can formulate the GEMs with corresponding graph filters.}
    \resizebox{0.9\textwidth}{!}{%
    \begin{tabular}{|l|l|l|l|l|}
    \hline
    Graph Embedding Models & Graph-shift filter $S$ & Polynomial Function $h(x)$ & Input Signal & Parameters $\Theta$  \\
    \hline
    GCN \cite{ICLR2017SemiGCN}  & $L^{sym}-I_n$ & $h(x) = x$ & X & Any \\
    \hline
    SGC \cite{sgc_icml19}  & $L^{sym}-I_n$ & $h(x) = x$ & X     & Any \\
    \hline
    ChebyNet \cite{Defferrard2016ChebNet} & $L^{sym}-I_n$ & $h(x) = \sum_{k=0}^{K}T_k(x)$ & X     & Any \\
    \hline
    LINE \cite{WWW2015Line} & $I_n - L^{rw}$ & $h(x) = x$ & $\frac{1}{b}I_n$ & $vol(G)D^{-1}$ \\
    \hline
    DeepWalk  \cite{perozzi2014deepwalk} & $I_n - L^{rw}$ & $h(x)=\sum_{k=0}^{K}x^k$ & $\frac{1}{b}I_n$ & $vol(G)D^{-1}$ \\
    \hline
    \end{tabular}
    \label{tab:allres}%
    }
    % \begin{flushleft}
    % \end{flushleft}
    \vspace{-3mm}
\end{table*}%

\subsection{Embedding Quality Measure $\mathcal{L}_{\textrm{attack}}$ of \textit{GF-Attack}}
% \revision{According to~\eqref{equ.problemdef-simp}, in order to avoid accessing the target model parameters $\Theta$, we can construct the restricted black-box attack loss $\mathcal{L}_{\textrm{attack}} (A', X)$ by attacking the quality of output embedding $Z = \mathscr{M}_{\Theta}(A', X)$. With the help of the tools from GSP, we have the following assumption:}
\revision{Now that we have the formulation~\eqref{equ.problemdef-simp} of the optimization problem under our general framework \textit{GF-Attack}, the next step is to design an effective measure for evaluating the quality of the output embeddings.
Recent works \cite{yang2015network,nar2019cross} demonstrate that the output embeddings of GEMs can have a very low rank.
Therefore, we establish the general measure of embedding quality in ~\eqref{equ.problemdef-simp} accordingly as a $T$-rank approximation problem~\cite{WSDM2018NetworkEmbedding}:}
\begin{align}
\revision{
\mathcal{L}_{\textrm{attack}}(A', X) =  \|\tilde{X'} - \tilde{X'}_T \|_F^2 = \|h(S')X - h(S')_TX \|_F^2, \notag
}
\end{align}
where $h(S')$ is the polynomial graph filter, $S'$ is the graph shift filter constructed from the perturbed adjacency matrix $A'$. $ h(S')_T$ is the $T$-rank approximation of $h(S')$. According to the low-rank approximation, $\mathcal{L}_{\textrm{attack}}(A', X)$ can be rewritten as:
\begin{align}
% \small{
\mathcal{L}_{\textrm{attack}}(A', X) &= \| \sum_{i = T + 1}^{n} \lambda'_{i} \mathbf{u}_{i}\mathbf{u}_{i}^{T}X \|^{2}_{F} \label{equ.low-rank of loss}\\
\notag & \revision{\leq  \sum_{i = T + 1}^{n}  \|\lambda'_{i}\|_{F}^{2} \|\mathbf{u}_{i}\|_{F}^{2} \|\mathbf{u}_{i}^{T} X\|_{F}^{2}}\\
& \leq \sum_{i = T + 1}^{n} {\lambda'_{i}}^{2} \cdot \sum_{i = T + 1}^{n} \|\mathbf{u}_{i}^{T}X\|_2^2,
\label{equ.loss}
\end{align}
where $n$ is the number of vertices. $h(S)=  U\Lambda U^{\text{T}}$ is the eigen-decomposition of the graph filter $h(S)$. $h(S)$ is a symmetric matrix. $\Lambda= diag(\lambda_1,\cdots,\lambda_n)$, $U= [\mathbf{u}_{1},\cdots,\mathbf{u}_{n}]$ are the eigenvalue and eigenvector of graph filter $\mathscr{H}$, respectively, in order of $ \lambda_{1} \geq \lambda_{2} \geq \dots \geq \lambda_{n}$. $\lambda'_{i}$ is the corresponding eigenvalue after perturbation. 

\revision{As the output embedding of a well-learned GEM has the desired low-rank property, we can view the training process of GEM as implicitly minimizing the attack loss $\mathcal{L}_{\textrm{attack}}$. On the opposite, for the attack purpose, we need to maximize $\mathcal{L}_{\textrm{attack}}$ for generating effective adversarial edges.
While $\| \sum_{i = T + 1}^{n} \lambda'_{i} \mathbf{u}_{i}\mathbf{u}_{i}^{T}X \|^{2}_{F}$ in Eq.~\eqref{equ.low-rank of loss} is hard to optimize, we can find its upper bound as in Eq.~\eqref{equ.loss}. Then during the generation of graph embeddings, the minimizing of this upper bound will be induced when the GEMs minimize the attack loss. Accordingly, the goal of adversarial attack can be maximizing the upper bound of the loss reversely, since~\eqref{equ.low-rank of loss} and~\eqref{equ.loss} generally have the same monotonicity \wrt $\lambda'_{i}$ as we show in the following Theorem~\ref{thm.monotonicity}:}
\begin{theorem}\label{thm.monotonicity}
\revision{
When all $\lambda'_{i}$ for $i \in [T+1, n]$ have the same signs, \eqref{equ.low-rank of loss} and~\eqref{equ.loss} are monotonically related \wrt $\lambda'$.
}
\end{theorem}
\begin{proof}
\revision{
We denote $f(\lambda'_{i}) = \| \sum_{i = T + 1}^{n} \lambda'_{i} \mathbf{u}_{i}\mathbf{u}_{i}^{T}X \|^{2}_{F}$ and $g(\lambda'_{i}) = \sum_{i = T + 1}^{n} {\lambda'_{i}}^{2} \cdot \sum_{i = T + 1}^{n} \|\mathbf{u}_{i}^{T}X\|_2^2$. Then for all non-negative $\lambda'_{i}$, it is easy to check that both $f(\lambda'_{i})$ and $g(\lambda'_{i})$ are non-decreasing \wrt $\lambda'_{i}$. Then for any pair of values, $\lambda_{i}^{'(a)}$ and $\lambda_{i}^{'(b)}$, if $f(\lambda_{i}^{'(a)}) \leq f(\lambda_{i}^{'(b)})$ holds then $g(\lambda_{i}^{'(a)}) \leq g(\lambda_{i}^{'(b)})$ also holds. By the definition, we can have that the two functions $f(\lambda'_{i})$ and $g(\lambda'_{i})$ are monotonically related. It is trivial to extend the same monotonicity for all non-positive $\lambda'_{i}$, which concludes the proof.
}
\end{proof}
\revision{
For both GCNs and sampling-based GEMs that are chosen as examples in this work, all $\lambda'_{i}$ for $i \in [T+1, n]$ can be chosen to have the same signs with a proper $T$, which reveals that Theorem~\ref{thm.monotonicity} generally holds in our framework.
Thus the restricted black-box adversarial attack loss ~\eqref{equ.problemdef-simp} under \textit{GF-Attack} framework is equivalent to optimize:}
% \begin{align}
% \notag\arg\max\limits_{A'} & \sum_{i = T + 1}^{n} {\lambda'_{i}}^{2} \cdot \sum_{i = T + 1}^{n} \|\mathbf{u}^{T}_{i}X\|_2^2,\\
% % \text{s.t.}& \; \| A' - A\| = 2\beta.
% \revision{
% \text{s.t.}\; } & \revision{h(S)=  U\Lambda U^{\text{T}}, S' = f(A'), \| A' - A\|_{0} = 2\beta.}
% \label{equ.attackall}
% \end{align}
\begin{align}
\revision{
\arg\max\limits_{\lambda'_{i}} \sum_{i = T + 1}^{n} {\lambda'_{i}}^{2} \cdot \sum_{i = T + 1}^{n} \|\mathbf{u}^{T}_{i}X\|_2^2.
}
\label{equ.attackall}
\end{align}
According to~\eqref{equ.attackall}, we can attack any GEM that can be described by the corresponding graph filter $\mathscr{H}$. Meanwhile, our general attack framework also provides a view of theoretical explanation on the transferability of adversarial examples created by~\cite{KDD2018Adversarial,ICLR2019Meta,icml2019adversarial}, since modifying edges in adjacency matrix $A$ implicitly perturbs the eigenvalues of graph filters. In the following, we will analyze two kinds of popular GEMs and aim to construct the corresponding adversarial attack losses under \textit{GF-Attack} according to~\eqref{equ.attackall}.

\subsection{GF-Attack on Graph Convolutional Networks (GCNs)}\label{sec.GF-GCN}
\subsubsection{Formulation of GCNs with the corresponding graph filter $\mathscr{H}$}
Graph Convolution Networks (GCNs) extend the definition of convolution to the irregular graph structure and learn a representation vector of a vertex with feature matrix $X$. Namely, the Fourier transform is generalized on graphs to define the convolution operation: $g_{\theta} \ast \mathbf{x} = U g_{\theta}(\Lambda) U^{T} \mathbf{x}$. To accelerate the calculation, ChebyNet \cite{Defferrard2016ChebNet} proposes a polynomial filter $g_\theta(\Lambda) = \sum_{k=0}^{K}\theta_k\Lambda^k$ and approximates $g_{\theta}(\Lambda)$ by a truncated expansion concerning the Chebyshev polynomials $T_{k}(x)$:

\begin{equation}%\label{FourierTrans}
g_{\theta'} \ast \mathbf{x} \approx \sum_{k = 0}^{K} \theta'_{k}T_{k}(\widetilde{L})\mathbf{x},
\end{equation}
where $\widetilde{L} = \frac{2}{\lambda_{\text{max}}}L - I_{n}$ and $\lambda_{\text{max}}$ is the largest eigenvalue of Laplacian matrix $L$. $\theta{'} \in \mathbb{R}^{K}$ is now the parameters of Chebyshev polynomials $T_{k}(x)$. 
\revision{$K$ denotes the $K_{\text{th}}$ order Chebyshev polynomial. }
Due to the natural connection between Fourier transform and signal processing, it's easy to formulate the loss for ChebyNet under \textit{GF-Attack}:
\begin{lemma}
The $K$-localized single-layer ChebyNet with activation function $\sigma(\cdot)$ and weight matrix $\Theta$ is equivalent to filter graph signal $X$ with a polynomial filter $\mathscr{H} = \sum_{k=0}^{K}T_k(S)$ with graph-shift filter $S = 2\frac{L^{sym}}{\lambda_{max}} - I_n$. $T_k(S)$ represents the Chebyshev polynomial of order $k$. Eq.~\eqref{equ.GF-Attack} can be rewritten as:
\begin{equation*}
    \notag\tilde{X} = \sum_{k=0}^{K}T_k(2\frac{L^{sym}}{\lambda_{max}} - I_n)X,
    \,\,\,\, X' =  \sigma(\tilde{X}\Theta).
\end{equation*}
\label{lemma.chebynet}
\end{lemma}
\begin{proof}
The $K$-localized single-layer ChebyNet with activation function $\sigma(\cdot)$ is $\sigma(\sum_{k=0}^{K}\theta'_{k}T_k(2\frac{L^{sym}}{\lambda_{max}} - I_n)X)$. Thus, we can directly write the graph-shift filter as $S = 2\frac{L^{sym}}{\lambda_{max}} - I_n$, and write the linear and shift-invariant filter as $\mathscr{H} = \sum_{k=0}^{K}T_k(S)$.
\end{proof}

GCN \cite{ICLR2017SemiGCN} constructs the layer-wise model by simplifying the ChebyNet with $K=1$, $\theta'_{0} = 1$ and $\theta'_{1} = -1$. Then the \emph{re-normalization trick} is used to avoid gradient exploding/vanishing:
\begin{eqnarray}
\label{Eq:gcn}
X^{(l+1)} &=& \sigma\left(\tilde{D}^{-\frac{1}{2}}\tilde{A} \tilde{D}^{-\frac{1}{2}}X^{(l)}\Theta^{(l)}\right),
\end{eqnarray}
where $\tilde{A} = A + I_n$ and $\tilde{D}_{ii} = \sum_{j}\tilde{A}_{ij}$. $\Theta^{(l)}$ are the parameters in the $l_{th}$ layer and $\sigma(\cdot)$ is an activation function.

SGC \cite{sgc_icml19} further utilizes a single linear transformation to achieve computationally efficient graph convolution, i.e., $\sigma(\cdot)$ in SGC is a linear activation function. \revision{We can formulate the multi-layer SGC in the sense of generalization of GSP, in favor of the construction of the corresponding attack loss under GF-Attack, through its theoretical connection to ChebyNet}:
\begin{corollary}\label{thm.sgc}
The $K$-layer SGC is equivalent to the $K$-localized single-layer ChebyNet with $K_{th}$ order polynomials of the graph-shift filter $S^{sym}= 2I_n - L^{sym}$. Eq.~\eqref{equ.GF-Attack} can be rewritten as:
% \begin{equation*}
% \[
\begin{align*}
    \tilde{X} = (2I_n - L^{sym})^{K}X, \,\,\,\, X' =  \sigma(\tilde{X}\Theta).
\end{align*}
% \]
% \end{equation*}
\end{corollary}
\begin{proof}
We can write the $K$-layer SGC as $(2I_n - L^{sym})^{K} X \Theta$. Since $\Theta$ are the learned parameters in the neural network, we can employ the reparameterization trick to use $(2I_{n} - L^{sym})^{K}$ to approximate the same order polynomials $\sum_{k=0}^{K}T_k(2I_n - L^{sym})$ with a new $\widetilde{\Theta}$. Then we rewrite the $K$-layer SGC by polynomial expansion as $\sum_{k=0}^{K}T_k(2I_n - L_{sym}) X \widetilde{\Theta}$. Therefore, we can directly write the graph-shift filter $S^{sym} = 2I_n - L^{sym}$ with the same linear and shift-invariant filter $\mathscr{H}$ as $K$-localized single-layer ChebyNet.
\end{proof}
Note that SGC and GCN are identical when $K=1$.
Even though the non-linearity disturbs the explicit expression of the graph-shift filter of multi-layer GCN, the spectral analysis from \cite{sgc_icml19} demonstrates that both GCN and SGC share similar graph filtering behavior. Thus, we extend the general attack loss from multi-layer SGC to multi-layer GCN under the non-linear activation function scenario. Our experiments confirm that the attack loss for multi-layer SGC also shows excellent performance on multi-layer GCN.

\subsubsection{\textit{GF-Attack} loss for SGC/GCN}\label{sec: GCN loss}
% \textbf{\textit{GF-Attack} loss for SGC/GCN.}
As stated in Corollary~\ref{thm.sgc}, the graph-shift filter $S$ of SGC/GCN is defined as $S^{sym} =  2I_n - L^{sym} = D^{-\frac{1}{2}}AD^{-\frac{1}{2}} + I_n = \hat{A} + I_n$, where $\hat{A}$ denotes the normalized adjacency matrix. Thus, for $K$-layer SGC/GCN, we can decompose the graph filter $\mathscr{H}$ as $\mathscr{H}^{sym} = (S^{sym})^K = U_{\hat{A}} (\Lambda_{\hat{A}} + I_{n})^{K} U_{\hat{A}}^{T}$, where $\Lambda_{\hat{A}}$ and $U_{\hat{A}}$ are the eigen-pairs of $\hat{A}$. The corresponding adversarial attack loss for $K_{th}$ order SGC/GCN can be written as:
\begin{equation}
\arg\max\limits_{A'} \sum_{i = T + 1}^{n} (\lambda'_{\hat{A'},i} + 1)^{2K} \cdot \sum_{i = T + 1}^{n} \|\mathbf{u}^{T}_{\hat{A},i}X\|_2^2,
\label{equ.GF-Attack-sym}
\end{equation}
where $\lambda'_{\hat{A'},i}$ refers to the $i_{th}$ largest eigenvalue of the perturbed normalized adjacency matrix $\hat{A'}$.

Directly calculating $\lambda'_{\hat{A'},i}$ from attacked normalized adjacency matrix $A'$ will need an eigen-decomposition operation, which is extremely time consuming. Therefore, we introduce the eigenvalue perturbation theory~\cite{Book1990Matrix} to fast estimate $\lambda'_{\hat{A'},i}$ in a linear time:

\begin{lemma}\label{thm:General_Eigen}
Let $A' = A + \Delta A$ be a perturbed version of $A$ by adding/removing edges and $\Delta D$ be the respective change in the degree matrix. $\lambda_{\hat{A},i}$ and $\mathbf{u}_{\hat{A},i}$ are the $i_{th}$ eigen-pair of eigenvalue and eigenvector of $\hat{A}$ and also solve the generalized eigen-problem $A\mathbf{u}_{\hat{A},i}=\lambda_{\hat{A},i} D\mathbf{u}_{\hat{A},i}$. Then the perturbed generalized eigenvalue $\lambda^{'}_{\hat{A},i}$ is approximately as:
\begin{align}
    \lambda'_{\hat{A'},i} \approx \lambda_{\hat{A},i} + \frac{ \mathbf{u}^{T}_{\hat{A},i}\Delta A\mathbf{u}_{\hat{A},i} - \lambda_{\hat{A},i}\mathbf{u}^{T}_{\hat{A},i}\Delta Du_{\hat{A},i} }{\mathbf{u}^{T}_{\hat{A},i} D \mathbf{u}_{\hat{A},i}}.
\label{equ:General_Eigen}
\end{align} 
\end{lemma}
\begin{proof}
\revision{Please kindly refer to~\cite{zhu2018high}.}
\end{proof}
\begin{remark}
With Theorem \ref{thm:General_Eigen}, we can directly derive the explicit formulation (Eq.~\eqref{equ:General_Eigen}) of $\lambda'_{\hat{A'}}$ perturbed by $\Delta A$ on the original adjacency matrix $A$.
\end{remark}

\textbf{Order $K$ irrelevant loss for SGC/GCN.} 
As shown in~\eqref{equ.GF-Attack-sym}, \textit{GF-Attack} should know (or assume) the order $K$ to perform the attack on the victim model. To further relax this constraint and make our framework for adversarial attack adapted to stricter RBA settings, we investigate the formulation of Eq.~\eqref{equ.GF-Attack-sym} without the impact from order $K$.

Since our aim is finding the proper $\lambda'_{\hat{A'},i}$ to maximize the loss, thus we can find the lower bound of Eq.~\eqref{equ.GF-Attack-sym} and maximize the lower bound correspondingly. Thus, the information from order $K$ can be omitted properly. Following this approach,
% To further make the loss of \textit{GF-Attack} fits the RBA setting, we investigate the formulation of Eq.~\ref{equ.GF-Attack-sym} without impact from order $K$. The core idea is since our aim is finding the proper $\lambda'_{\hat{A'},i}$ to maximize the loss, thus we can find the lower bound of Eq.~\eqref{equ.GF-Attack-sym} and maximize the lower bound correspondingly. 
We figure out the relationship between the order $K$ and the lower bound of Eq.~\eqref{equ.GF-Attack-sym}: 

\begin{theorem}\label{thm:GCN-no-order-K}
The eigenvalues of $\hat{A'}$ are denoted as $1 \geq \lambda'_{\hat{A'},1} \geq \lambda'_{\hat{A'},2} \geq \dots \geq \lambda'_{\hat{A'},n} \geq -1$ . Suppose a large enough $T$ is chosen to ensure the smallest $n-T$ eigenvalues, the optimization variables of Eq.~\eqref{equ.GF-Attack-sym}, all negative from $[-1,0)$, then Eq.~\eqref{equ.GF-Attack-sym} is a monotonically decreasing function of $K$, and the corresponding adversarial attack loss for $K_{th}$ order SGC/GCN is the lower bound for losses with orders less than $K$.
\end{theorem}

\begin{proof}
Since $K$ is irrelevant to the eigenvector part $\sum_{i = T + 1}^{n} \|\mathbf{u}^{T}_{\hat{A},i}X\|_2^2$, \revision{our aim is to find} the lower bound of $f(K) = \sum_{i = T + 1}^{n} (\lambda'_{\hat{A'},i} + 1)^{2K}$. Taking the derivative of $f(x)$ directly, we can have $f'(K) = \sum_{i = T + 1}^{n} (\lambda'_{\hat{A'},i} + 1)^{2K} \cdot 2\ln(\lambda'_{\hat{A'},i} + 1)$. \revision{As we ensure that the choice} of $T$ 
is large enough to make $\lambda'_{\hat{A'}, i} \in [-1,0)$, thus $2\ln(\lambda'_{\hat{A'},i} + 1) < 0$ and $f'(K) < 0$. This makes the attack loss function~\eqref{equ.GF-Attack-sym} for $K_{th}$ order SGC/GCN a monotonically decreasing function of $K$, which indicates that it is the lower bound for the losses with orders less than $K$.
\end{proof}

\begin{remark}
From Theorem~\ref{thm:GCN-no-order-K}, instead of knowing the number of the layer $K$, we can conduct effective attacks for the target SGC/GCN models by optimizing the lower bound of the adversarial attack loss function~\eqref{equ.GF-Attack-sym}. Therefore, we can choose a relatively large $K$ in the loss function~\eqref{equ.GF-Attack-sym} for SGC/GCN to perform effective attacks in practice.
\end{remark}

%From Theorem~\ref{thm:GCN-no-order-K}, the adversarial attack loss function~\eqref{equ.GF-Attack-sym} of $K=4$ is the lower bound for $K={1,2,3}$. 

%Instead of knowing the number of the layer $K$, we can conduct effective attacks for the target SGC/GCN models by optimizing the lower bound of the adversarial attack loss function~\eqref{equ.GF-Attack-sym}

%it will conduct effective attacks for the target SGC/GCN models with layers from $1$ to $4$.
%Therefore, we can choose a relatively large order $K$ of the loss function~\eqref{equ.GF-Attack-sym} for SGC/GCN in practice, even without the information of the actual layer used in target models.

\subsection{GF-Attack on Sampling-based GEMs}\label{sec: DW loss}

\subsubsection{Formulation of Sampling-based GEMs with the corresponding graph filter $\mathscr{H}$}

Sampling-based GEMs learns vertex representations according to the sampled vertices~\cite{grover2016node2vec}, vertex sequences~\cite{li2017deepcas}, or network motifs~\cite{ribeiro2017struc2vec}. For instance, LINE~\cite{WWW2015Line} with the second order proximity intends to learn two graph representation matrices $X'$, $Y'$ by maximizing the NEG loss of the skip-gram model:
\begin{equation}
\mathcal{L} = \sum_{i=1}^{|\mathcal{V}|} \sum_{j=1}^{|\mathcal{V}|} A_{i,j} \Big(\log \sigma(x'^{T}_{i} y'_{j}) + b\mathbb{E}_{j' \sim P_{N}}[\log \sigma(-x'^{T}_{i} y'_{j})] \Big),
\end{equation}
where $x'_{i}$, $y'_{i}$ are rows of $X'$, $Y'$, respectively. $\sigma$ is the activation function and chosen as sigmoid here. $b$ is the negative sampling parameter. $P_{N}$ denotes the noise distribution generating negative samples. Meanwhile, DeepWalk \cite{perozzi2014deepwalk} adopts the similar loss function except that $A_{i,j}$ is replaced with an indicator function indicating whether vertices $v_i$ and $v_j$ are sampled in the same sequence within the given context window-size $K$. Most of sampling-based GEMs only consider the structural information and ignore the feature matrix $X$. The output representation matrix is purely learned from the graph topology.

From the perspective of sampling-based GEMs, the embedded matrix is obtained by generating a training corpus for the skip-gram model from an adjacency matrix or a set of random walks. \citet{WSDM2018NetworkEmbedding} shows that Point-wise Mutual Information (PMI) matrices are implicitly factorized in the sampling-based embedding approaches. It indicates that LINE/DeepWalk can be rewritten into a matrix factorization form:
\begin{lemma}\label{lemma.deepwalk}\cite{WSDM2018NetworkEmbedding}
Given the context window-size $K$ and the number of negative sample $b$, the result of DeepWalk in matrix form is equivalent to factorize the matrix:
\begin{equation}
M = \log{\Big(\frac{\text{vol}(G)}{bK}(\sum_{k=1}^{K}(D^{-1}A)^{k}){D}^{-1}\Big)},
\label{equ.deepwalk}
\end{equation}
where $\text{vol}(G) = \sum_{i}\sum_{j}A_{ij} = \sum_{i}D_{ii}$ denotes the volume of graph $G$. And LINE can be viewed as a special case of DeepWalk with $K=1$.
\end{lemma}

\revision{For the proof of Lemma}~\ref{lemma.deepwalk}, please kindly refer to \cite{WSDM2018NetworkEmbedding}.
Inspired by this insight, we prove that LINE can be viewed from a GSP manner as well:
\begin{theorem}
LINE is equivalent to filter a graph signal $X = \frac{1}{b}I_{n}$ with a polynomial filter $\mathscr{H}$ and fixed parameters $\Theta=\text{vol}(G)D^{-1}$. $\mathscr{H}=S$ is constructed by graph-shift filter $S^{rw}=I_n - L^{rw}$. Eq.~\eqref{equ.GF-Attack} can be rewritten as:
\begin{align}
    \notag\tilde{X} &= \frac{1}{b}(I_{n} - L^{rw})D^{-1}I_{n}, \,\,\,\,
    \notag X'         = log(\text{vol}(G)\tilde{X}).
\end{align}
\label{thm.line}
\end{theorem}
Note that LINE is formulated from an optimized unsupervised NEG loss of a skip-gram model. 
Therefore, the parameter $\Theta$ and the value of the NCG loss are fixed with given graph signals.

We can extend Theorem~\ref{thm.line} to DeepWalk since LINE can be viewed as a $1$-window special case of DeepWalk:
\begin{corollary}
The output of $K$-window DeepWalk with $b$ negative samples is equivalent to filtering a set of graph signals $X = \frac{1}{b}I_{n}$ with given parameters $\Theta=\text{vol}(G)D^{-1}$. Eq.~\eqref{equ.GF-Attack} can be rewritten as:
\begin{equation}
    \notag\tilde{X} = \frac{1}{bK}\sum_{k=1}^{K}(I_{n} - L^{rw})^{k}D^{-1}I_{n}, \,\,\,\,
    \notag X'         = log(\text{vol}(G)\tilde{X}). 
\end{equation}
\label{col.deepwalk}
\end{corollary}

\begin{proof}[Proof of Theorem \ref{thm.line} and Corollary \ref{col.deepwalk}]
With Lemma~\ref{lemma.deepwalk}, we can explicitly write DeepWalk as
$\exp{(M)} = \frac{\text{vol}(G)}{b}(\sum_{k=1}^{K} \frac{1}{K}(I_n - L^{rw})^{k}D^{-1}I_n)$. Therefore, we can directly have the explicit expression of Eq.~\eqref{equ.GF-Attack} on LINE/DeepWalk.
\end{proof}

As stated in Corollary~\ref{col.deepwalk}, the graph-shift filter $S$ of DeepWalk is defined as $S^{rw} =  I_n - L^{rw} = D^{-1}A =  D^{-\frac{1}{2}}\hat{A}D^{\frac{1}{2}}$. Therefore, the graph filter $\mathscr{H}$ of the $K$-window DeepWalk can be decomposed as $\mathscr{H}^{rw} = \frac{1}{K}\sum_{k=1}^{K}(S^{rw})^{k}$, which satisfies $ \mathscr{H}^{rw} D^{-1}= D^{-\frac{1}{2}}U_{\hat{A}}(\frac{1}{K}\sum_{k=1}^{K}\Lambda_{\hat{A}}^k)U_{\hat{A}}^{T}D^{-\frac{1}{2}}$.

\subsubsection{\textit{GF-Attack} loss for LINE/DeepWalk}
% \textbf{\textit{GF-Attack} loss for LINE/DeepWalk.}

Since multiplying $D^{-\frac{1}{2}}$ in \textit{GF-Attack} loss brings extra complexity, \cite{WSDM2018NetworkEmbedding} provides us a way to well approximate the perturbed $\lambda'_{\mathscr{H}^{rw}D^{-1}}$ without this term:

\begin{lemma}\label{lem:Eigen_Bound}\cite{WSDM2018NetworkEmbedding}
Let $\hat{A} = U \Lambda U^{T}$ and $\mathscr{H}_{rw}=\sum_{r=1}^{K}S_{rw}^r$ be the graph-shift filter of DeepWalk. The decreasing order $s^{th}$ eigenvalue of $\mathscr{H}_{rw}$ are bounded as: $\lambda_{rw,s} \leq \frac{1}{d_{min}}|\frac{1}{K}\sum_{r = 1}^{K}\lambda^{r}_{\pi_{s}}|$, where $\{\pi_{1},\pi_{2},\dots,\pi_{n}\}$ is a permutation of $\{1,2,\dots,n\}$ ensuring the eigenvalue $\lambda$ in the non-increasing order and $d_{min}$ is the smallest degree in $A$. Then the smallest eigenvalue of $\mathscr{H}_{rw}$ is bounded as:
% \vspace{-1.5ex}
\begin{align*}
    \lambda_{min}(\mathscr{H}_{rw}) \geq \frac{1}{d_{min}}\lambda_{min}(U (\frac{1}{K}\sum_{r = 1}^{K}\Lambda^{r}) U^{T}).
\end{align*}
\end{lemma}
% \vspace{-1.5ex}

\revision{For the proof of Lemma}~\ref{lem:Eigen_Bound}, please kindly refer to \cite{WSDM2018NetworkEmbedding}.
Inspired by Lemma~\ref{lem:Eigen_Bound}, we can find that both the magnitude of eigenvalues and smallest eigenvalue of $\mathscr{H}^{rw}D^{-1}$ are always well-bounded.
Thus we have $\lambda'_{\mathscr{H}^{rw}D^{-1}} \approx \frac{1}{d_{\min}}\lambda'_{U_{\hat{A}}(\frac{1}{K}\sum_{k=1}^{K}\Lambda_{\hat{A}}^k)U_{\hat{A}}^{T}}$. Therefore, the corresponding adversarial attack loss of  $K_{th}$ order DeepWalk can be written as:
\begin{equation}
\arg\max\limits_{A'} \sum_{i = T + 1}^{n} (\frac{1}{d_{\min}}|\frac{1}{K}\sum_{k = 1}^{K}\lambda'^{k}_{\hat{A'},i}|)^{2} \cdot \sum_{i = T + 1}^{n} \|\mathbf{u}^{T}_{\hat{A},i}X\|_2^2.
\label{equ.GF-Attack-rw}
\end{equation}

\begin{corollary}
From Lemma~\ref{lemma.deepwalk}, we can easily extend Eq.~\eqref{equ.GF-Attack-rw} for DeepWalk to LINE by setting $K=1$, since LINE is a special case of DeepWalk with $K=1$.
\end{corollary}

% When $K=1$, Eq.~\eqref{equ.GF-Attack-rw} becomes the adversarial attack loss of LINE. 
Similarly, Theorem \ref{thm:General_Eigen} is utilized to estimate $\lambda'_{\hat{A'}}$ in the loss of LINE/DeepWalk. 

\textbf{Order $K$ irrelevant loss for LINE/DeepWalk.} Similar to the strategy we employ on the order $K$ irrelevant adversarial attack loss (Eq.~\eqref{equ.GF-Attack-sym}) for GCNs, we can also relax the constraint of assuming the window-size $K$ when performing the attack with loss Eq.~\eqref{equ.GF-Attack-rw}. 
% Concretely,  we can get the relation between order $K$ and the lower bound of Eq.~\eqref{equ.GF-Attack-rw}:
\revision{More specifically, the following Theorem~\ref{thm:DW-no-order-K} establishes the relationship:}

%With the same strategy, the relationship between the order $K$ and the lower bound of Eq.~\eqref{equ.GF-Attack-rw} is similar to Eq.~\eqref{equ.GF-Attack-sym} but more complicated:

\begin{theorem}\label{thm:DW-no-order-K}
Finding the lower bound for objective function~\eqref{equ.GF-Attack-rw} of order $K$ \revision{is equivalent to find} the lower bound for
\begin{align}
    f(K) = \sum_{i = T + 1}^{n} \frac{1}{K^2} (\sum_{k = 1}^{K}\lambda'^{k}_{\hat{A'},i})^{2}.
\end{align}
The smallest eigenvalue of $\hat{A'}$ other that $-1$ is denoted as $\lambda'_{\hat{A'},\min}$. Suppose a large enough $T$ is chosen to make sure the smallest $n-T$ eigenvalues, the optimization variables of Eq.~\eqref{equ.GF-Attack-rw}, all negative from $[-1,0)$, then as long as $K$ satisfies
\begin{align}
    K \geq \sqrt{\frac{n-T}{\min{f(K)}}} \frac{1}{1+\lambda'_{\hat{A'},\min}},
\end{align}
the corresponding attack loss with $K_{th}$ order LINE/DeepWalk is the lower bound for losses with orders smaller than $K$, where $\min{f(K)}$ is the minimum of $f(K)$, and $\lambda'_{\hat{A'},\min}$ is the minimum eigenvalue in series $\lambda'_{\hat{A'}, i}$ except -1s.
\end{theorem}

\begin{proof}
Since $\lambda'_{\hat{A'}, i} \in [-1,0)$, we can directly have Eq.~\eqref{equ.GF-Attack-rw} equal to $\sum_{i = T + 1}^{n} \frac{1}{K^2} (\sum_{k = 1}^{K}\lambda'^{k}_{\hat{A'},i})^{2} \cdot \frac{1}{d^2_{\min}} \cdot \sum_{i = T + 1}^{n} \|\mathbf{u}^{T}_{\hat{A},i}X\|_2^2$. By eliminating the parts that irrelevant to order $K$, our aim is equivalent to finding the lower bound of $f(K) = \sum_{i = T + 1}^{n} \frac{1}{K^2} (\sum_{k = 1}^{K}\lambda'^{k}_{\hat{A'},i})^{2}$. We conduct category discussion \wrt $K$ here:
% \begin{itemize}
    % \item 
    
    \noindent{{\bf{When $K$ is even},}} \ie, $K = 2z, z\in \mathbb{N}$, for the top $m$ smallest $\lambda'_{\hat{A'}, i} = -1$, we have $f(K)$ the following:
    \begin{align}
    f(K) &= \sum_{i = T + 1}^{n} \frac{1}{K^2} (\sum_{k = 1}^{K}\lambda'^{k}_{\hat{A'},i})^{2} \label{equ.MS_0}\\
    \notag & = \revision{\sum_{i = T + 1}^{n-m} \frac{1}{K^2} (\sum_{k = 1}^{K}\lambda'^{k}_{\hat{A'},i})^{2}} + \revision{\sum_{i = n - m + 1}^{n} \frac{1}{K^2} (\sum_{k = 1}^{K}{(-1)^k})^{2}} \\
    \notag & = \revision{\sum_{i = T + 1}^{n-m} \frac{1}{K^2} (\sum_{k = 1}^{K}\lambda'^{k}_{\hat{A'},i})^{2}} + \revision{\sum_{i = n - m + 1}^{n} \frac{1}{K^2} * 0} \\
    & = \revision{\sum_{i = T + 1}^{n-m} \frac{1}{K^2} (\sum_{k = 1}^{K}\lambda'^{k}_{\hat{A'},i})^{2}}.\label{equ.MS_1}
    \end{align}
    Since $\lambda'_{\hat{A'}, i} \neq -1$ in Eq.~\eqref{equ.MS_1}, we can have \revision{the following from Maclaurin Series:}
    \begin{align}
    f(K) \notag &= \sum_{i = T + 1}^{n - m} \frac{1}{K^2} (\sum_{k = 1}^{K}\lambda'^{k}_{\hat{A'},i})^{2}\\
    \notag &< \revision{\sum_{i = T + 1}^{n - m} \frac{1}{K^2} (\sum_{k = 1}^{K}(-\lambda'_{\hat{A'},i})^{k})^{2}}\\
    \notag &< \sum_{i = T + 1}^{n - m} \frac{1}{K^2} (\frac{1}{1 + \lambda'_{\hat{A'},i}})^{2}
    \end{align}
    For the minimum eigenvalue $\lambda'_{\hat{A'},\min}$ in series $\lambda'_{\hat{A'}, i}$ except -1s, because $\frac{1}{1 + \lambda'_{\hat{A'},i}} < \frac{1}{1 + \lambda'_{\hat{A'},\min}}$, the following inequality holds
    \begin{align*}
        f(K) &< \revision{\sum_{i = T + 1}^{n - m} \frac{1}{K^2} (\frac{1}{1 + \lambda'_{\hat{A'},\min}})^{2}} \\
        &= \revision{(n-m-T) \frac{1}{K^2} (\frac{1}{1 + \lambda'_{\hat{A'},\min}})^{2}.}
    \end{align*}
    Assume $\min{f(K)}$ is the minimum of $f(K)$ for $K \in {1,2,\dots,K}$, which can be easily obtained in practice, it follows that 
    \begin{align*}
       \min{f(K)} \leq f(K) < \revision{(n-m-T) \frac{1}{K^2} (\frac{1}{1 + \lambda'_{\hat{A'},\min}})^{2}.}
    \end{align*}
    % \xhdr{Add discussion on minf(K)-m}
    In order to find a $K$ that satisfies the adversarial attack loss~\eqref{equ.GF-Attack-rw} for $K_{th}$ order LINE/DeepWalk is the lower bound, we need to have
    \begin{align*}
       \min{f(K)} * K^2 \geq \revision{(n-m-T) (\frac{1}{1 + \lambda'_{\hat{A'},\min}})^{2}}
    \end{align*}
    Given $\min{f(K)} \geq 0$, which is obvious from Eq.~\eqref{equ.MS_0} as It turns out that as long as $K$ satisfies:
    \begin{align}
        K \geq \sqrt{\frac{n-m-T}{\min{f(K)}}} \frac{1}{1+\lambda'_{\hat{A'},\min}}, \label{equ.EvenK}
    \end{align}
    the adversarial attack loss~\eqref{equ.GF-Attack-rw} for $K_{th}$ order LINE/DeepWalk is the lower bound for the losses with orders less than $K$.
    % \item 
    
    \noindent{{\bf{When $K$ is odd},}} \ie, $K = 2z+1$, for the top $m$ smallest $\lambda'_{\hat{A'}, i} = -1$, we have $f(K)$ the following:
    \begin{align}
    f(K) \notag & = \revision{\sum_{i = T + 1}^{n-m} \frac{1}{K^2} (\sum_{k = 1}^{K}\lambda'^{k}_{\hat{A'},i})^{2}} + \revision{\sum_{i = n - m + 1}^{n} \frac{1}{K^2} (\sum_{k = 1}^{K}{(-1)^k})^{2}} \\
    \notag & = \revision{\sum_{i = T + 1}^{n-m} \frac{1}{K^2} (\sum_{k = 1}^{K}\lambda'^{k}_{\hat{A'},i})^{2}} + \revision{\sum_{i = n - m + 1}^{n} \frac{1}{K^2}(-1)^2} \\
    & = \revision{\sum_{i = T + 1}^{n-m} \frac{1}{K^2} (\sum_{k = 1}^{K}\lambda'^{k}_{\hat{A'},i})^{2}} + \frac{m}{K^2}.\label{equ.MS_2}
    \end{align}
    Then we can have the similar result as the situation $K$ is even from Maclaurin Series:
    \begin{align*}
        &f(K) < \revision{(n-m-T) \frac{1}{K^2} (\frac{1}{1 + \lambda'_{\hat{A'},\min}})^{2} + \frac{m}{K^2}} \\
        &= \left( (n-T) (\frac{1}{1 + \lambda'_{\hat{A'},\min}})^{2} - m (\frac{1}{1 + \lambda'_{\hat{A'},\min}})^{2} + m \right)\frac{1}{K^2}
    \end{align*}
    Since $(\frac{1}{1 + \lambda'_{\hat{A'},\min}})^{2} > 1$, then $-m(\frac{1}{1 + \lambda'_{\hat{A'},\min}})^{2} < -m$, thus we have 
    \begin{align*}
        f(K) &< \big( (n-T) (\frac{1}{1 + \lambda'_{\hat{A'},\min}})^{2} - m + m \big)\frac{1}{K^2} \\
        &= (n-T) \frac{1}{K^2} (\frac{1}{1 + \lambda'_{\hat{A'},\min}})^{2} .
    \end{align*}
    With the same analysis as $K$ is even, we can have the desired condition for an odd $K$ as
    \begin{align}
        K \geq \sqrt{\frac{n-T}{\min{f(K)}}} \frac{1}{1+\lambda'_{\hat{A'},\min}}. \label{equ.OddK}
    \end{align}
    While $n-T \geq n-m-T$, combining the result from~\eqref{equ.EvenK}, we can have the overall desired condition for $K$ is~\eqref{equ.OddK}, which concludes the proof.
% \end{itemize}
\end{proof}
\begin{remark}
Similar to \textit{GF-Attack} on GCNs,  by choosing a relatively large order $K$ of the loss~\eqref{equ.GF-Attack-rw} for LINE/DeepWalk in practice, we can effectively attack the target LINE/DeepWalk model without the knowledge about the orders (window-size).
\end{remark}
%Therefore, by choosing a relatively large order $K$ of the loss~\eqref{equ.GF-Attack-rw} for LINE/DeepWalk in practice, we can effectively attack the target LINE/DeepWalk model without knowledge about orders (window-size).
%the same strategy that, choosing a relatively large order $K$ of the loss~\eqref{equ.GF-Attack-rw} for LINE/DeepWalk in practice, can effectively attack the target LINE/DeepWalk model without knowledge about orders (window-size).

\begin{algorithm}[!tb]
\caption{Algorithm in \textit{GF-Attack} under the RBA setting} \label{alg:Framework}
% \small{
\begin{algorithmic}[1] %1 means each row has a number

\REQUIRE ~~\\ %Input
adjacency matrix $A$; 
feature matrix $X$; 
target vertex $t$; \\
number of top-$T$ smallest singular values/vectors selected $T$; order of graph filter $K$;
fixed budget $\beta$.\\

\ENSURE ~~\\ %Output
Perturbed adjacency matrix $A'$.

\STATE Initial the candidate flips set as $\mathcal{C} = \{(v, t)|v \neq t \}$, eigenvalue decomposition of $\hat{A} = U_{\hat{A}}\Lambda_{\hat{A}}U_{\hat{A}}^{T}$;

\FOR {$(v, t) \in \mathcal{C}$}

\STATE Approximate $\Lambda'_{\hat{A}}$ resulting by removing/inserting edge $(v, t)$ via Eq.~\eqref{equ:General_Eigen};
\STATE {Update ${Score_{(v, t)}}$ from loss Eq.~\eqref{equ.GF-Attack-sym} or Eq.~\eqref{equ.GF-Attack-rw} \wrt the assumption on the type of victim model;}

\ENDFOR
\STATE $\mathcal{C}_{sel}$ $\gets$ edge flips with top-$\beta$ $Score$;
\STATE $A' \gets A \pm \mathcal{C}_{sel}$;
\RETURN{$A'$}
\end{algorithmic}
% }
\end{algorithm}

\subsection{Attack Algorithm}
Based on the general attack loss, the goal of our adversarial attack is to misclassify a target vertex $t$ from an attributed graph $G(\mathcal{V},\mathcal{E})$ given a downstream vertex classification task. We start by defining the candidate flips then the general attack loss is responsible for scoring the candidates.

Directly solving Eq.~\eqref{equ.GF-Attack-sym} or Eq.~\eqref{equ.GF-Attack-rw} is complex and time-consuming.
Though the search space for graph-structured data is discrete, searching over full graphs is still $\mathcal{O}(n^2)$, which also could be pretty complex for large graphs. To alleviate this issue, we first adopt the hierarchical strategy in \cite{ICML2018Adversarial} to decompose the single edge selection into two ends of this edge in practice.
\revision{
Then we let the candidate set $\mathcal{C}$ for edge selection contains all vertices (edges and non-edges) directly accessory to the target vertex, i.e. $\mathcal{C} = \{(v, t)|v \neq t \}$, which is consistent with~\cite{ICML2018Adversarial,icml2019adversarial}. Intuitively, the further away the vertices from target $t$, the less influence they impose on $t$.
Meanwhile, experiments in \cite{KDD2018Adversarial,icml2019adversarial} also show that their candidates from $\mathcal{C} = \{(v, t)|v \neq t \}$ can do significantly more damage compared to candidate flips chosen from the other parts of the graph. Thus, in our experiments we also choose to restrict our candidates within the same choice of set $\mathcal{C} = \{(v, t)|v \neq t \}$.
}

Overall, for a given target vertex $t$, we establish the targeted attack by sequentially calculating the corresponding \textit{GF-Attack} loss \wrt graph-shift filter $S$ for each flip in the candidate set as scores. Then with a fixed budget $\beta$, the adversarial attack is accomplished by selecting flips with top-$\beta$ scores as perturbations on the adjacency matrix $A$ of the clean graph. Details of the algorithm in \textit{GF-Attack} under RBA setting \revision{are} depicted in Algorithm~\ref{alg:Framework}.

\begin{table*}[!t]
\centering
\caption{Summary of the change in classification accuracy (in percent) compared to the clean/original graph. Single edge perturbation under the RBA and poisoning settings. Lower is better. We use \textit{Cheby} for ChebyNet and \textit{DW} for DeepWalk here to save space. \label{tab:results single edge}}
\resizebox{\textwidth}{!}{%
\begin{tabular}{ l c c c c c c c c c c c c c c c}
\toprule
    Dataset & \multicolumn{5}{c}{Cora} & \multicolumn{5}{c}{Citeseer} & \multicolumn{5}{c}{Pubmed}  \\
\cmidrule(lr){2-6}\cmidrule(l){7-11}\cmidrule(l){12-16}
Models & GCN & SGC & Cheby & DW & LINE & GCN & SGC & Cheby & DW & LINE & GCN & SGC & Cheby & DW & LINE\\
% \hline
(unattacked) &   80.20 & 78.82  & 80.33 & 77.23 & 76.75	& 72.50 & 69.68 & 70.96 & 69.74 & 65.15 & 80.40 & 80.21 & 79.45 & 78.69 & 72.12 \\
\hline
\textit{Random} & -1.81 &-2.01 & -2.30 & -1.84 & -2.61 & -1.57 & -1.74  & -1.92 & -1.44 & -1.13 & -2.01 & -2.18 & -1.28 & -1.95 & -1.34 \\
% \hline
\revision{\textit{Degree}} & -3.50 & -6.06  & -5.59 & -2.91 & -4.59 & -4.17 & -4.24 & -4.14 & -7.55 & -8.35 & -3.20 & -3.91 & -3.68 & -2.28 & -8.41 \\
% \hline
\textit{RL-S2V} & -4.10 & -5.12  & -6.48 & -4.52 & -5.39 & -4.05 & -4.08  & -4.55 & -11.13 & -10.05 & -5.64 & -6.71 & -4.46 & -5.10 & -12.21\\
% \hline
\textit{$\mathcal{A}_{class}$} & -3.89 & -6.54 & -8.10 & -8.63 & -7.12 & -5.42 & -6.14 & -5.96 & \textbf{-13.24} & -9.47 & -4.34 & -4.55 & -5.92 & -4.56 & -11.98 \\
%\hline
\midrule
\textit{GF-Attack} & \textbf{-5.56}  & \textbf{-7.09} & \textbf{-9.10} & \textbf{-9.95} & \textbf{-9.74} & \textbf{-8.47} & \textbf{-9.04} & \textbf{-8.06} & -12.38 & \textbf{-10.91} & \textbf{-7.06} & \textbf{-7.20} & \textbf{-7.64} & \textbf{-7.14} & \textbf{-13.26}\\
\bottomrule
\end{tabular}
}
\end{table*}

\section{Experiments}\label{sec.exp}
\textbf{Datasets.}
We evaluate our approach on three real-world datasets: Cora, Citeseer, and Pubmed. In all three citation network datasets, vertices are documents with corresponding bag-of-words features and edges are citation links. The data preprocessing settings follow the benchmark setup in \cite{ICLR2017SemiGCN}.
Only the largest connected component (LCC) is considered to be consistent with~\cite{KDD2018Adversarial}.
% A statistical overview of datasets is given in Table~\ref{tab:dataset statistics}.
For a statistical overview of datasets, please kindly refer to~\cite{KDD2018Adversarial}.

\textbf{Baselines.}
\revision{In the current literature}, few studies strictly follow the restricted black-box attack setting. They utilize the additional information to help construct the attackers, such as labels \cite{KDD2018Adversarial}, gradients \cite{ICML2018Adversarial}, etc. Therefore, we compare four baselines with our proposed attack framework under the RBA setting as follows:
\begin{itemize}[noitemsep,topsep=0pt,parsep=0pt,partopsep=0pt]
\item \textit{Random} \cite{ICML2018Adversarial}:  for each perturbation, randomly choosing insertion or removing of an edge in graph $G$. We report averages over 10 different seeds to alleviate the influence of randomness.
\item \textit{Degree} \cite{CIKM2012Gelling}:  for each perturbation, inserting or removing an edge based on degree centrality, which is equivalent to the sum of degrees in original graph $G$.
\item \textit{RL-S2V} \cite{ICML2018Adversarial}: a reinforcement learning-based attack method, which learns the generalizable attack policy for GCN under the RBA scenario.
\item \textit{$\mathcal{A}_{class}$} \cite{icml2019adversarial}: a matrix perturbation theory based black-box attack method designed for DeepWalk. Then \textit{$\mathcal{A}_{class}$} evaluates the targeted attacks on vertex classification by learning a logistic regression.
\end{itemize}

\textbf{Target Models.}
To validate the generalization ability of \textit{GF-Attack}, we choose four popular GEMs: GCN \cite{ICLR2017SemiGCN}, SGC \cite{sgc_icml19}, DeepWalk \cite{perozzi2014deepwalk} and LINE \cite{WWW2015Line} for evaluation. GCN and SGC are GCNs and the others are sampling-based GEMs. 
\revision{Considering that both GCN and SGC are fixed-filter GCNs, we additionally choose a representative parameterized-filter variant, ChebyNet~\cite{Defferrard2016ChebNet}, as a victim model to further evaluate the effectiveness of our framework. The attack loss for ChebyNet is consistent with GCN due to their theoretical connection as we analyzed in Section~\ref{sec.GF-GCN}. For ChebyNet, we set the order of Chebyshev polynomials $K$ as $2$.}
For DeepWalk, we set the window-size as $5$. For both LINE and DeepWalk, the number of negative sampling in skip-gram is set as $1$, and the embedding dimension is chosen as $32$. A logistic regression classifier is connected to the output embeddings of sampling-based methods for classification. Without other specification, all GCNs contain two layers.

\textbf{Attack Configuration.}
A small budget $\beta$ is applied to regulate all the attackers. To make this attack task more challenging, the budget $\beta$ is set to 1. Specifically, the attacker is limited to only adding/deleting a single edge given a target vertex $t$. For our method, we set the parameter $T$ in our general attack model as $n - T = 128$, which means that we choose the top-$T$ smallest eigenvalues for $T$-rank approximation in the embedding quality measure. Unless otherwise 
indication, the order of graph filter in \textit{GF-Attack} model is set as $K=2$. Following the setting in \cite{KDD2018Adversarial}, we split the graph into labeled (20\%) and unlabeled vertices (80\%). Further, the labeled vertices are splitted into equal parts for training and validation. The labels and classifier are invisible to the attacker due to the RBA setting. The attack performance is evaluated by the decrease of vertex classification accuracy following \cite{ICML2018Adversarial}. \revision{Without otherwise specification, the attack is conducted on GCNs under the evasion setting and on sampling-based GEMs under the poisoning setting.}

\begin{figure*}[htbp]
\centering
\subfigure {\includegraphics[width=0.24\linewidth]{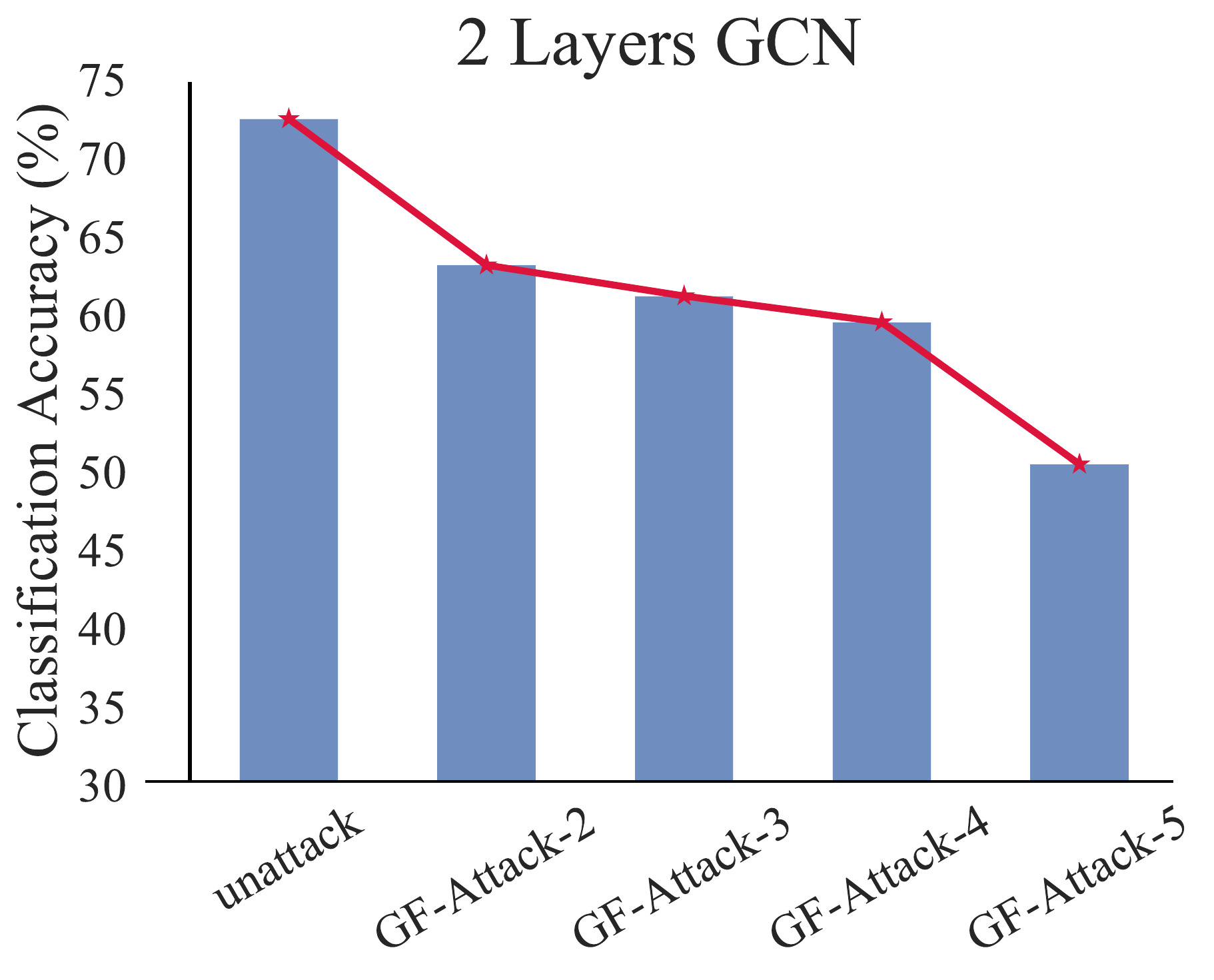}}
\subfigure {\includegraphics[width=0.24\linewidth]{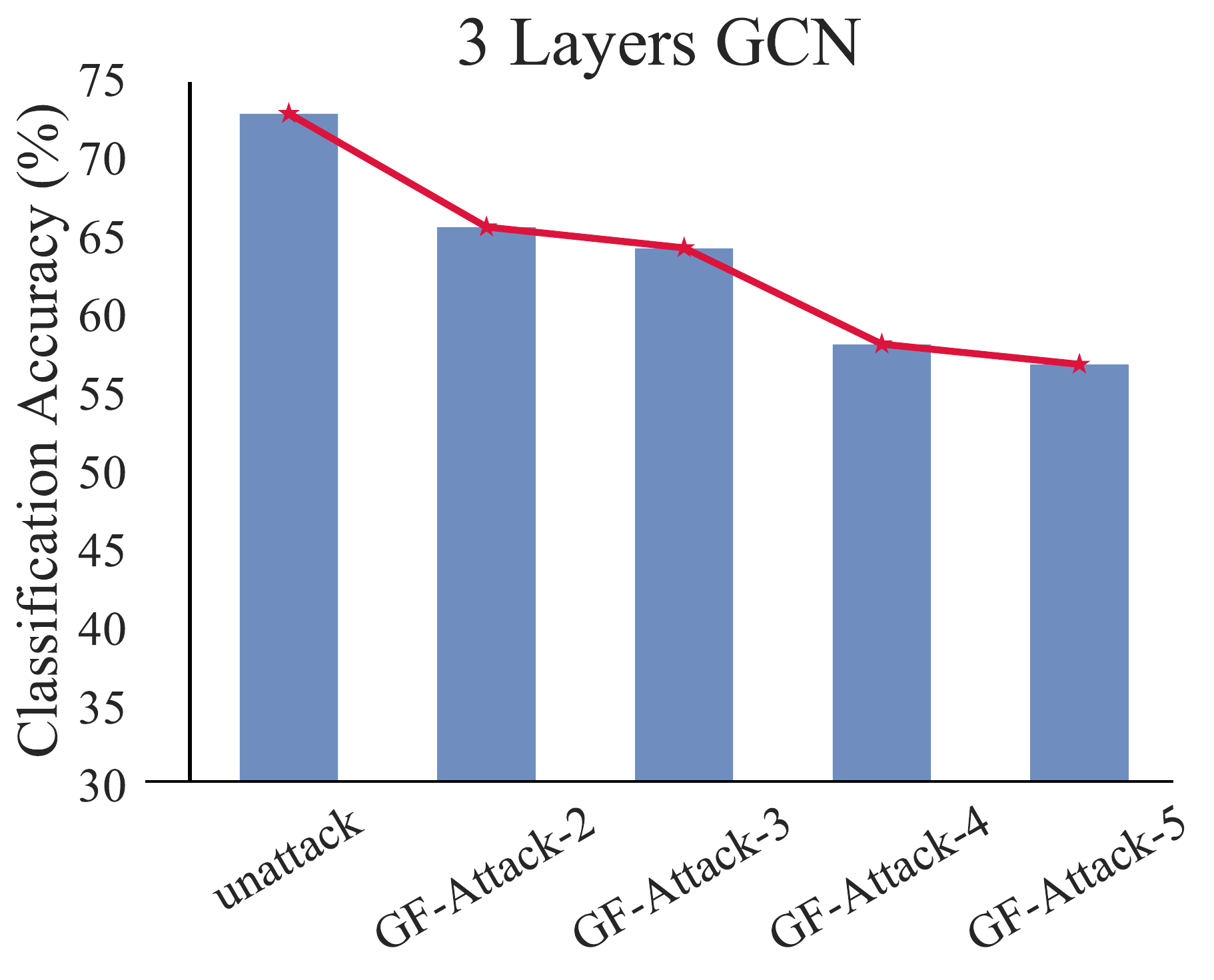}}
\subfigure {\includegraphics[width=0.24\linewidth]{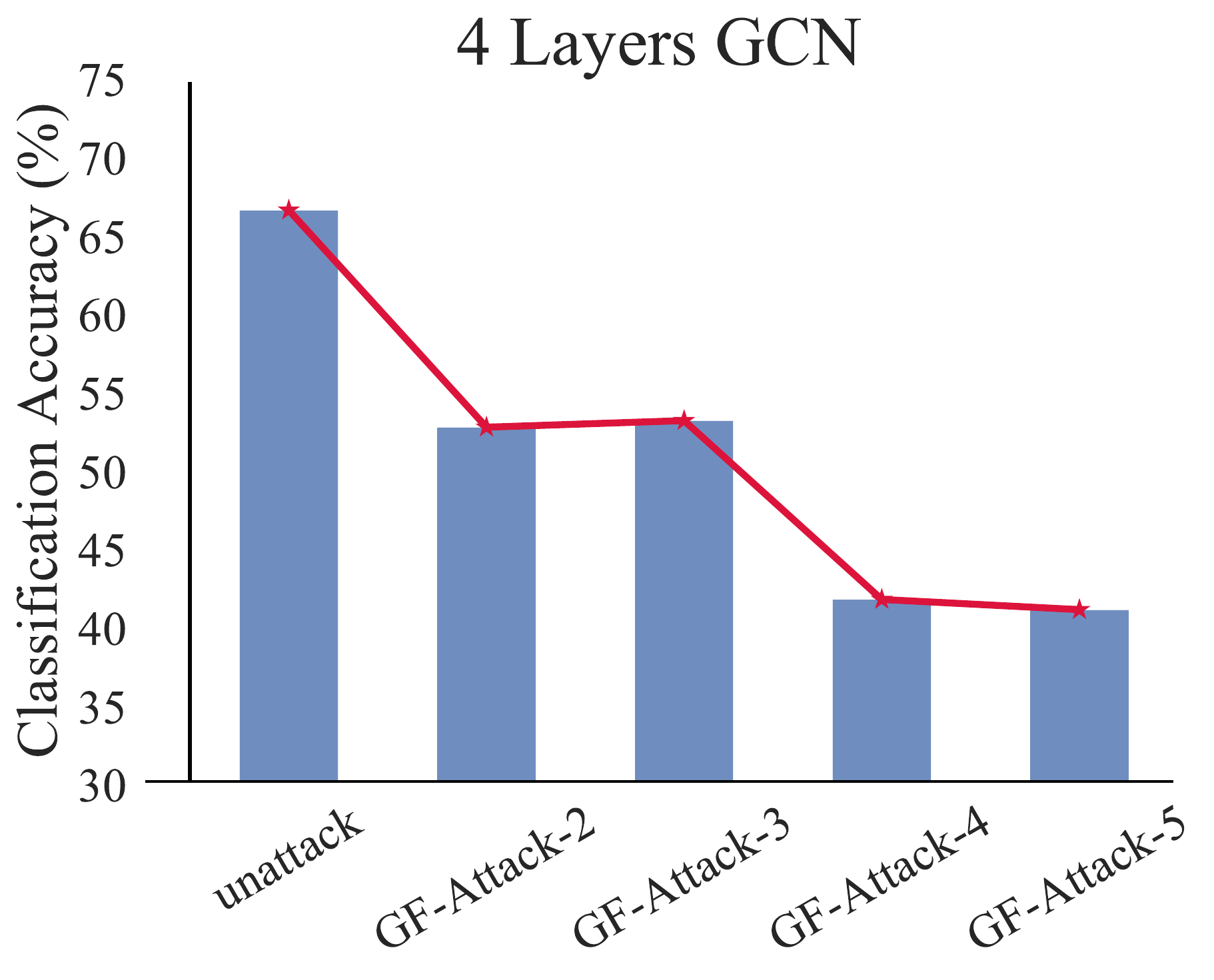}}
\subfigure {\includegraphics[width=0.24\linewidth]{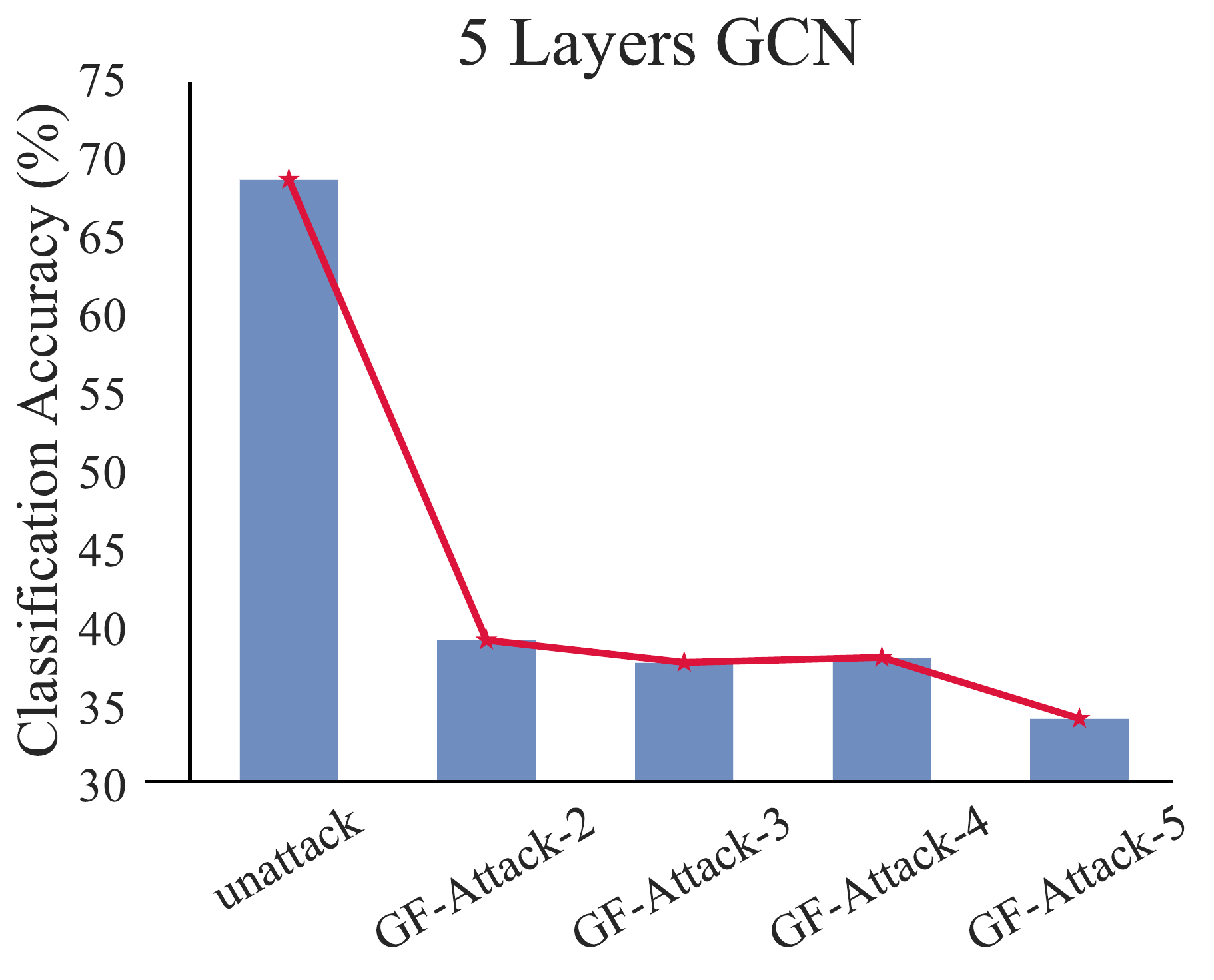}}
\caption{Comparison of the classification performance between order $K$ in \textit{GF-Attack} (x-axis) and the number of layers in GCN. Lower is better.}
	\label{fig:layer vs order GCN}
\end{figure*}

\begin{figure*}[htbp]
\centering
\subfigure {\includegraphics[width=0.24\linewidth]{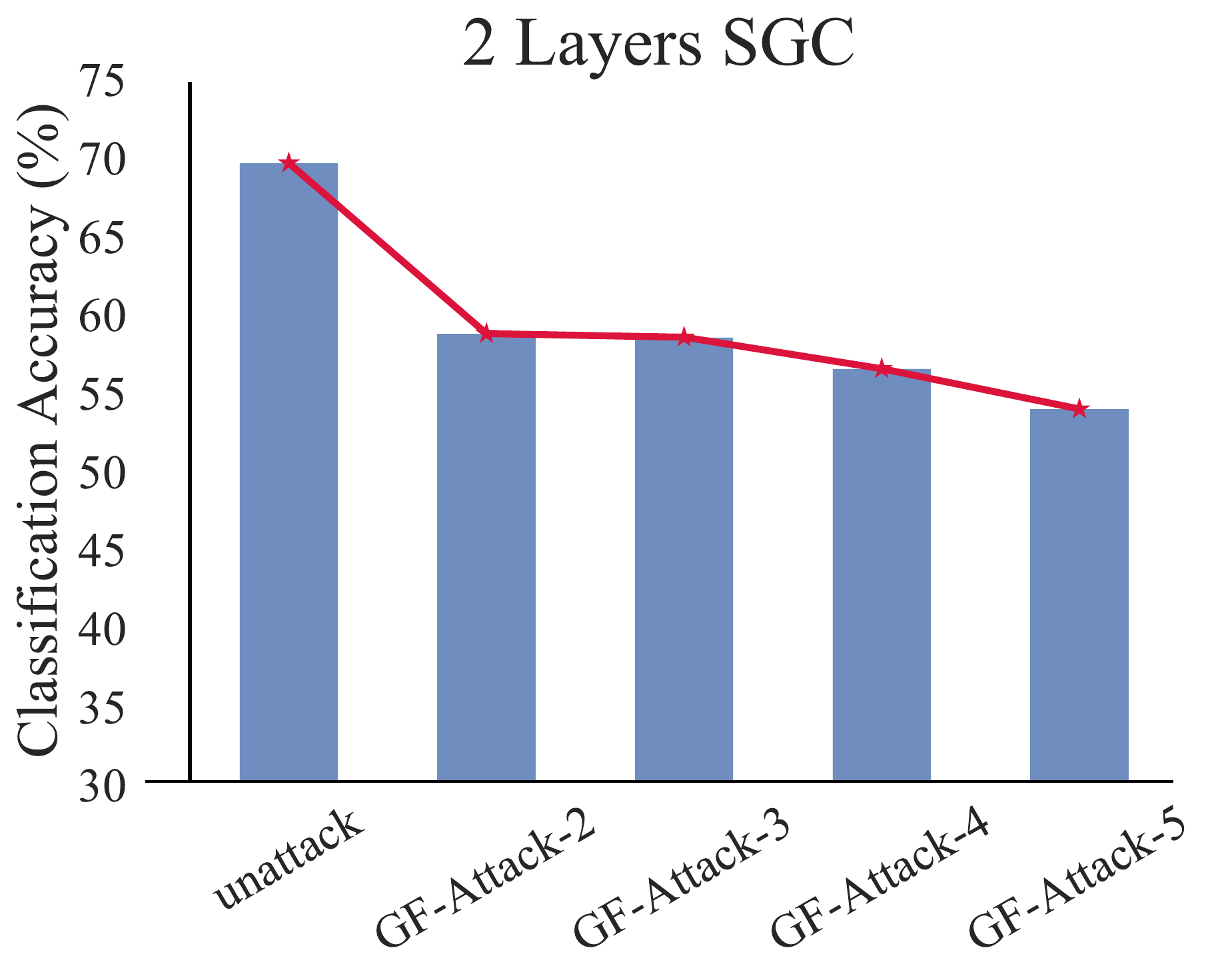}}
\subfigure {\includegraphics[width=0.24\linewidth]{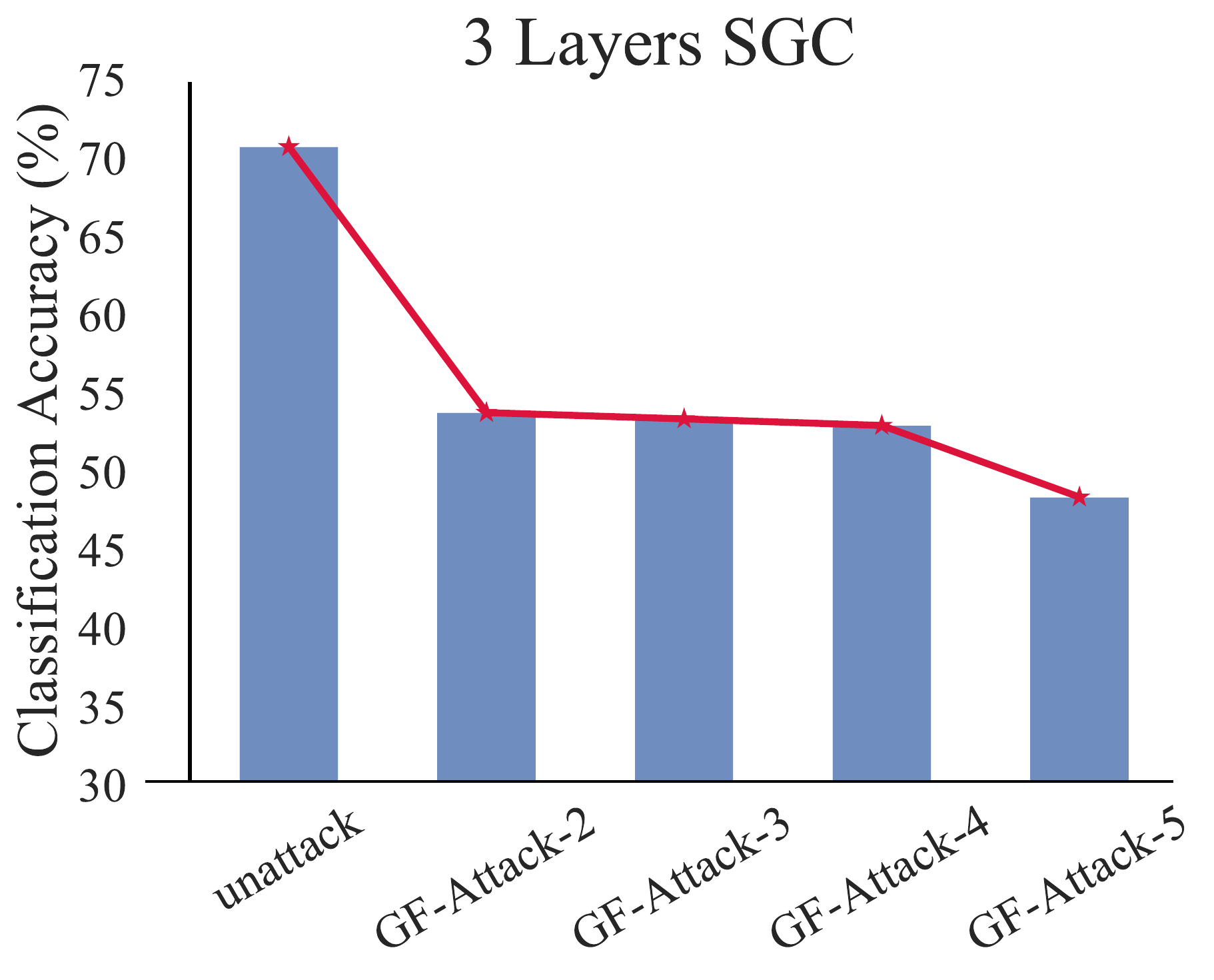}}
\subfigure {\includegraphics[width=0.24\linewidth]{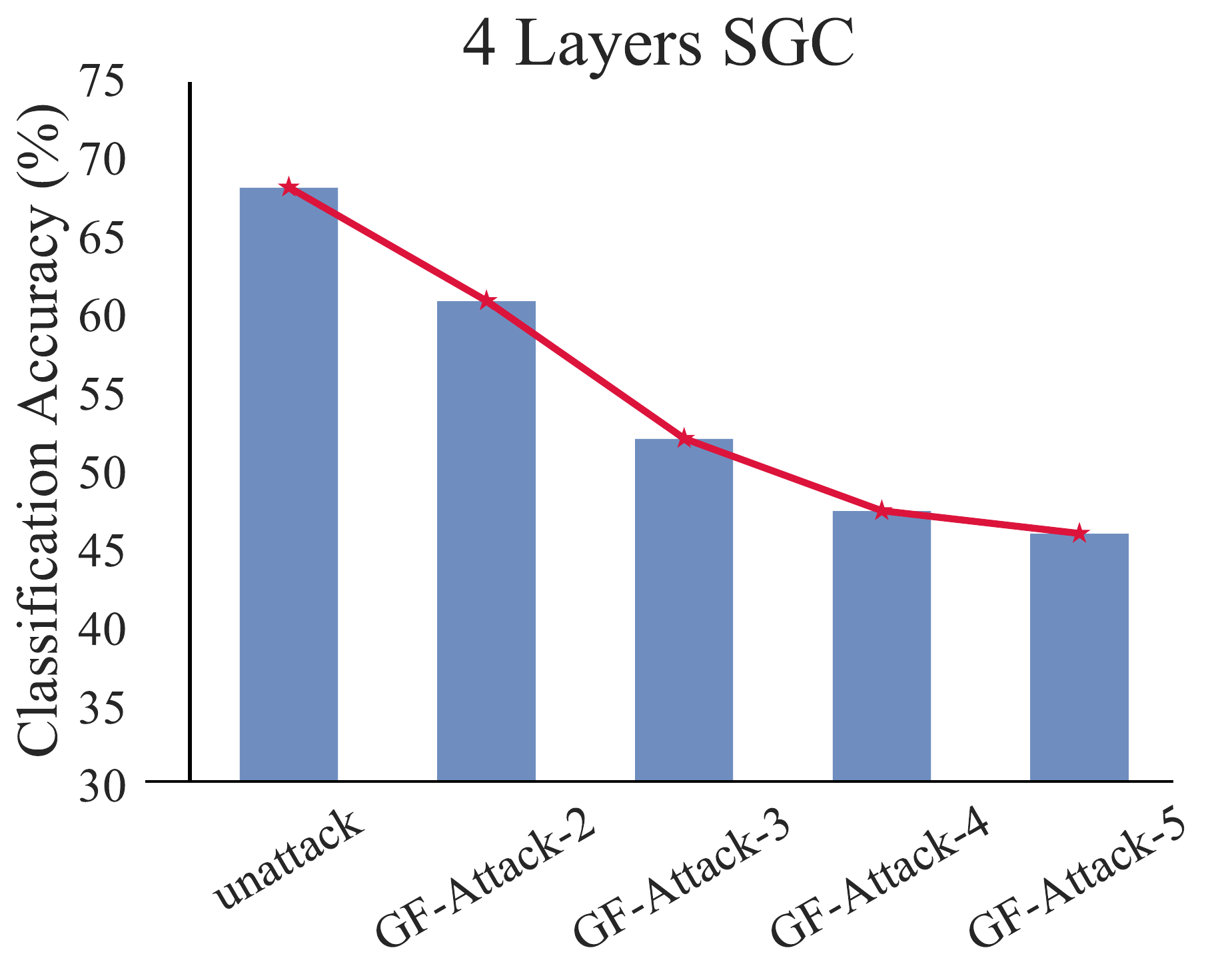}}
\subfigure {\includegraphics[width=0.24\linewidth]{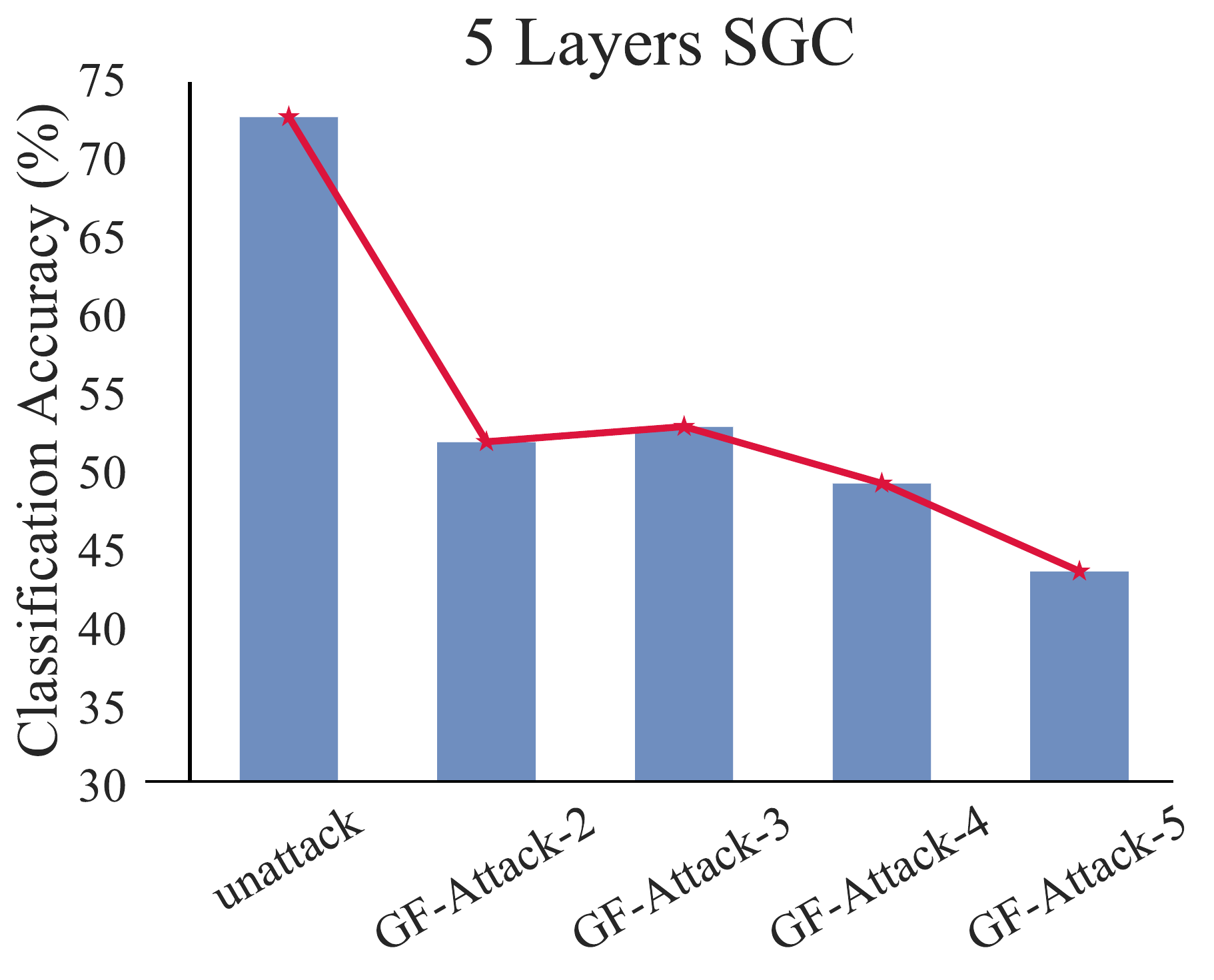}}
\caption{Comparison of the classification performance between order $K$ of \textit{GF-Attack} (x-axis) and the number of layers in SGC. Lower is better.}
	\label{fig:layer vs order SGC}
\end{figure*}

\begin{figure*}[htbp]
\centering
\subfigure {\includegraphics[width=0.195\linewidth]{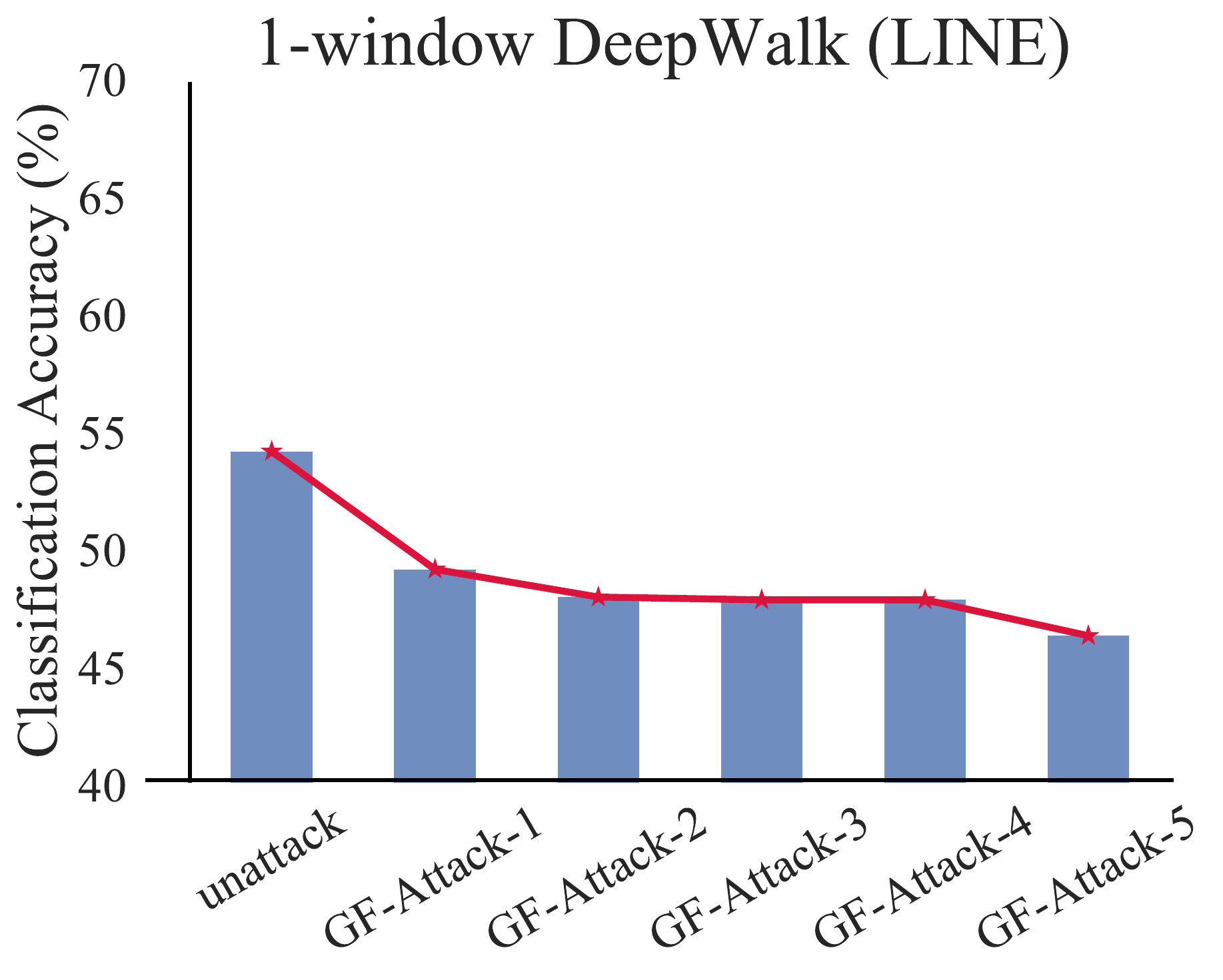}}
\subfigure {\includegraphics[width=0.195\linewidth]{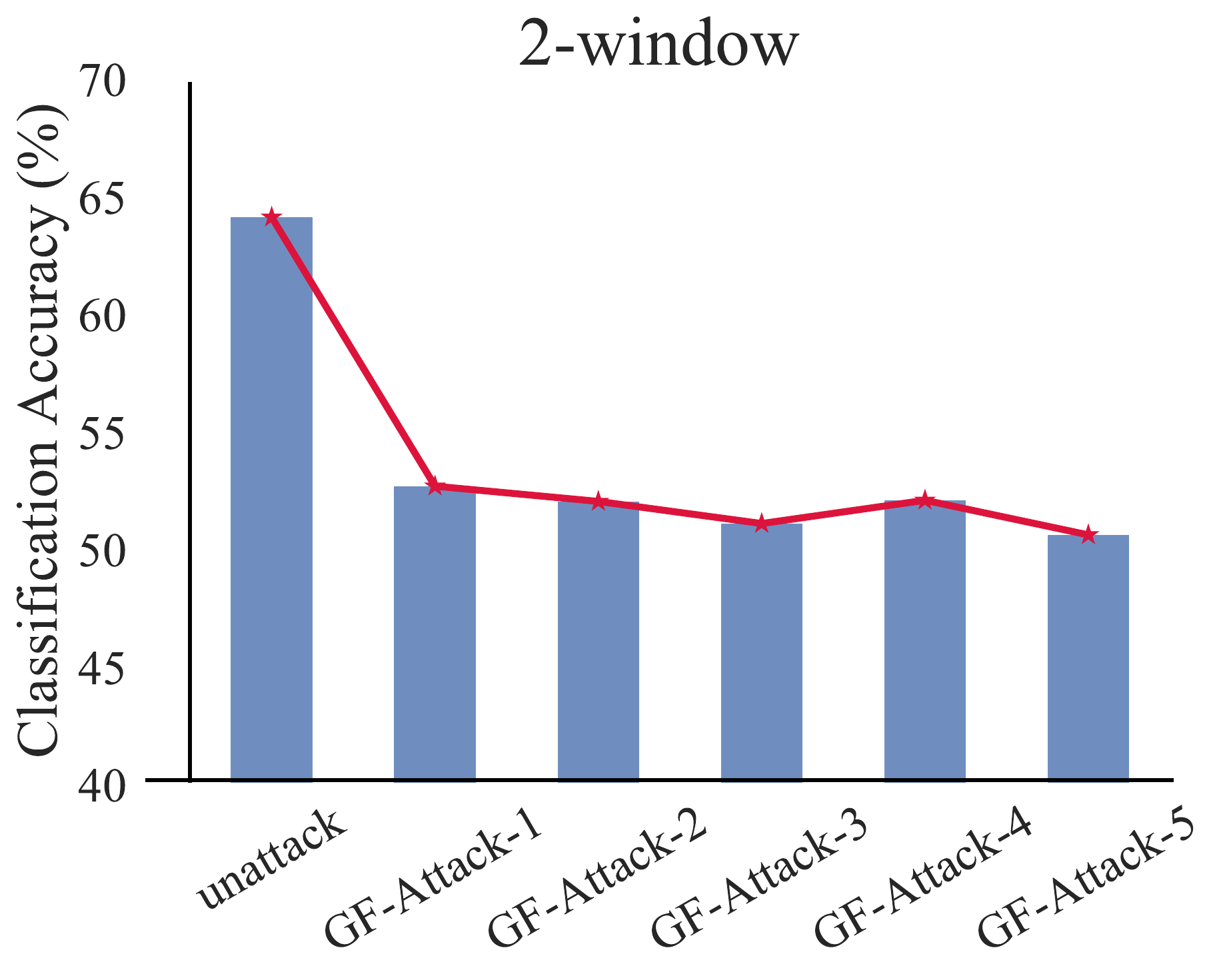}}
\subfigure {\includegraphics[width=0.195\linewidth]{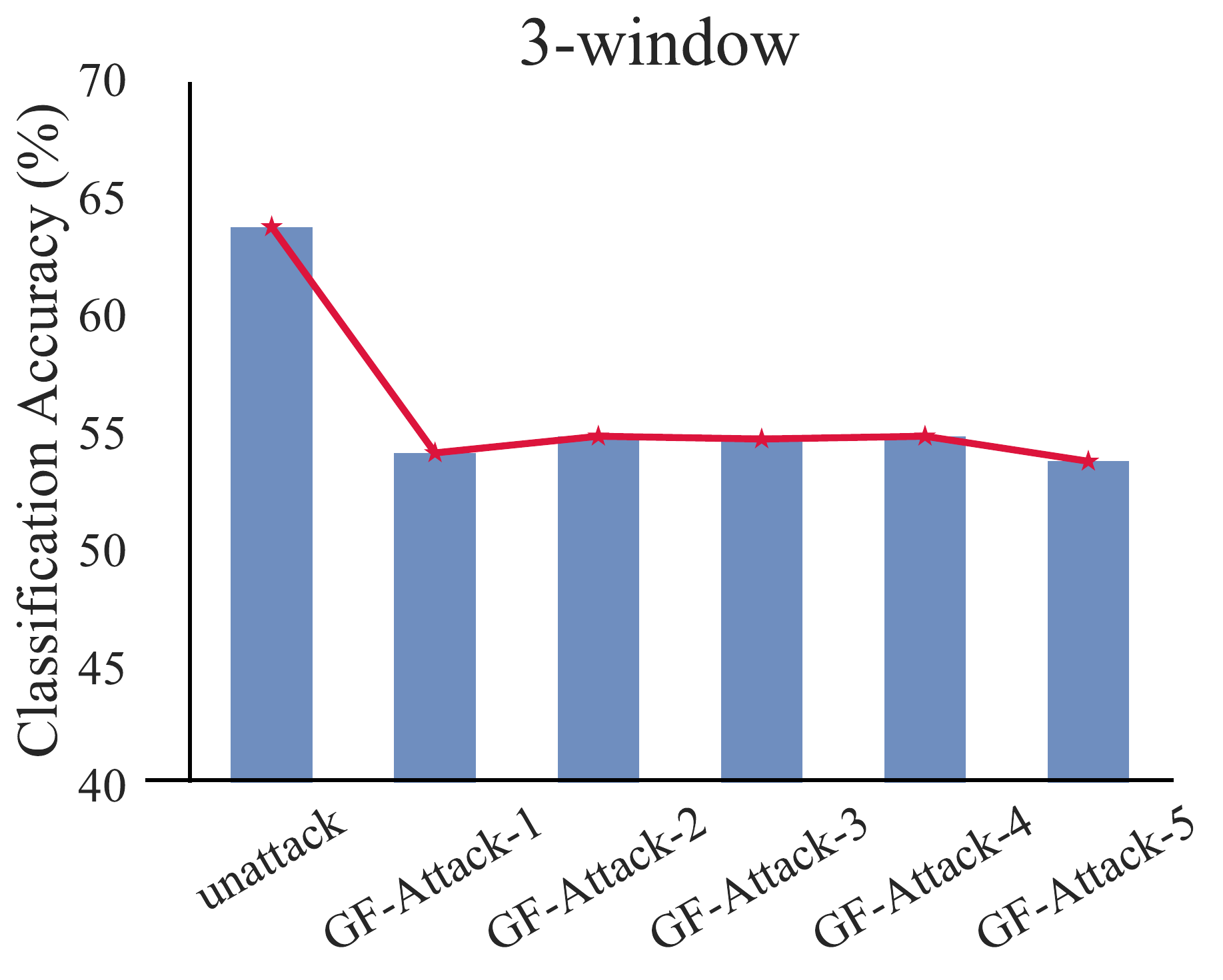}}
\subfigure {\includegraphics[width=0.195\linewidth]{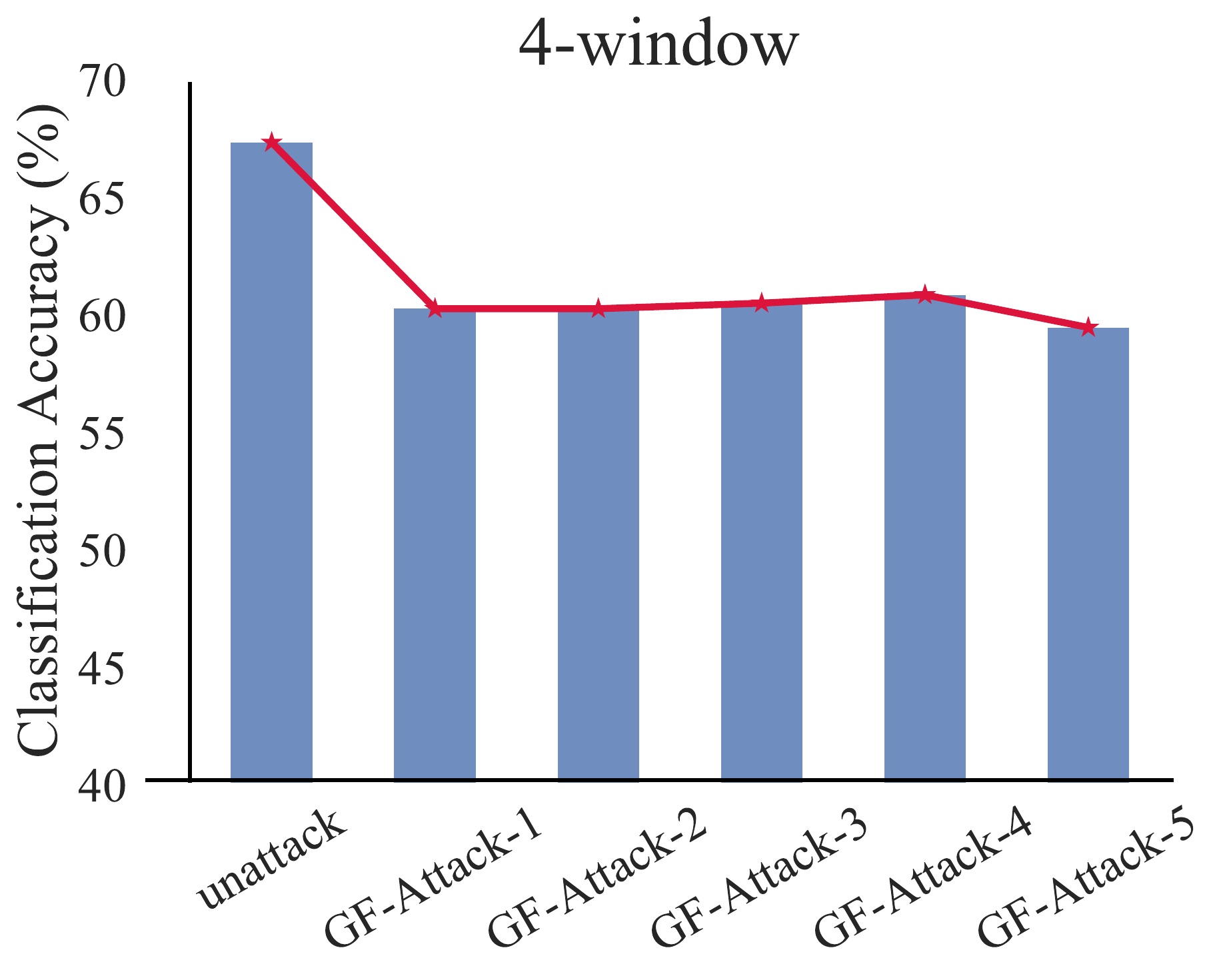}}
\subfigure {\includegraphics[width=0.195\linewidth]{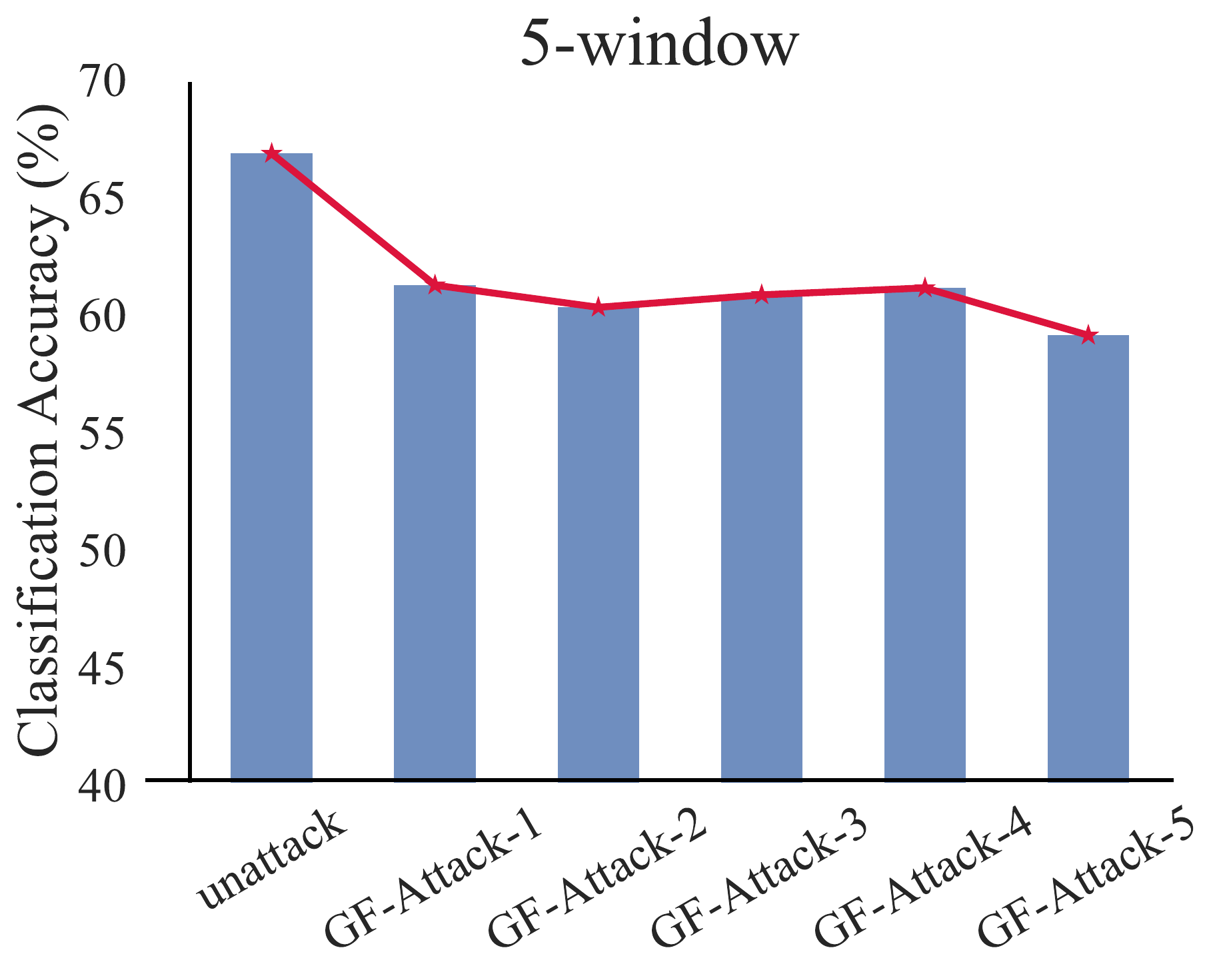}}
\caption{Comparison of the classification performance between order $K$ of \textit{GF-Attack} (x-axis) and the number of window-size in DeepWalk. Lower is better. Note that LINE is a 1-window special case of DeepWalk.}
	\label{fig:layer vs order DW}
\end{figure*}

\begin{figure*}[htbp]
\centering
\includegraphics[width=0.33\textwidth]{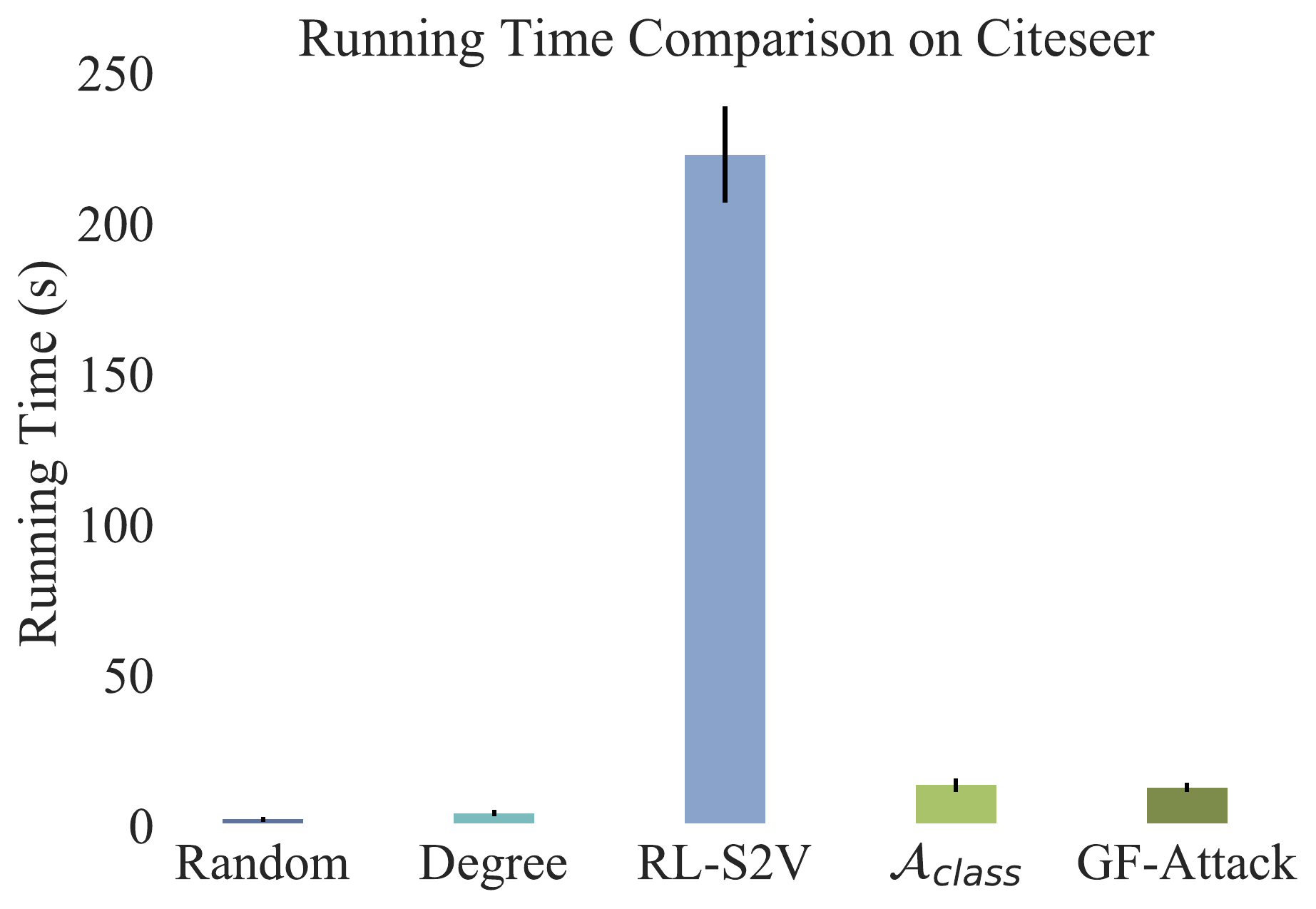}
\includegraphics[width=0.33\textwidth]{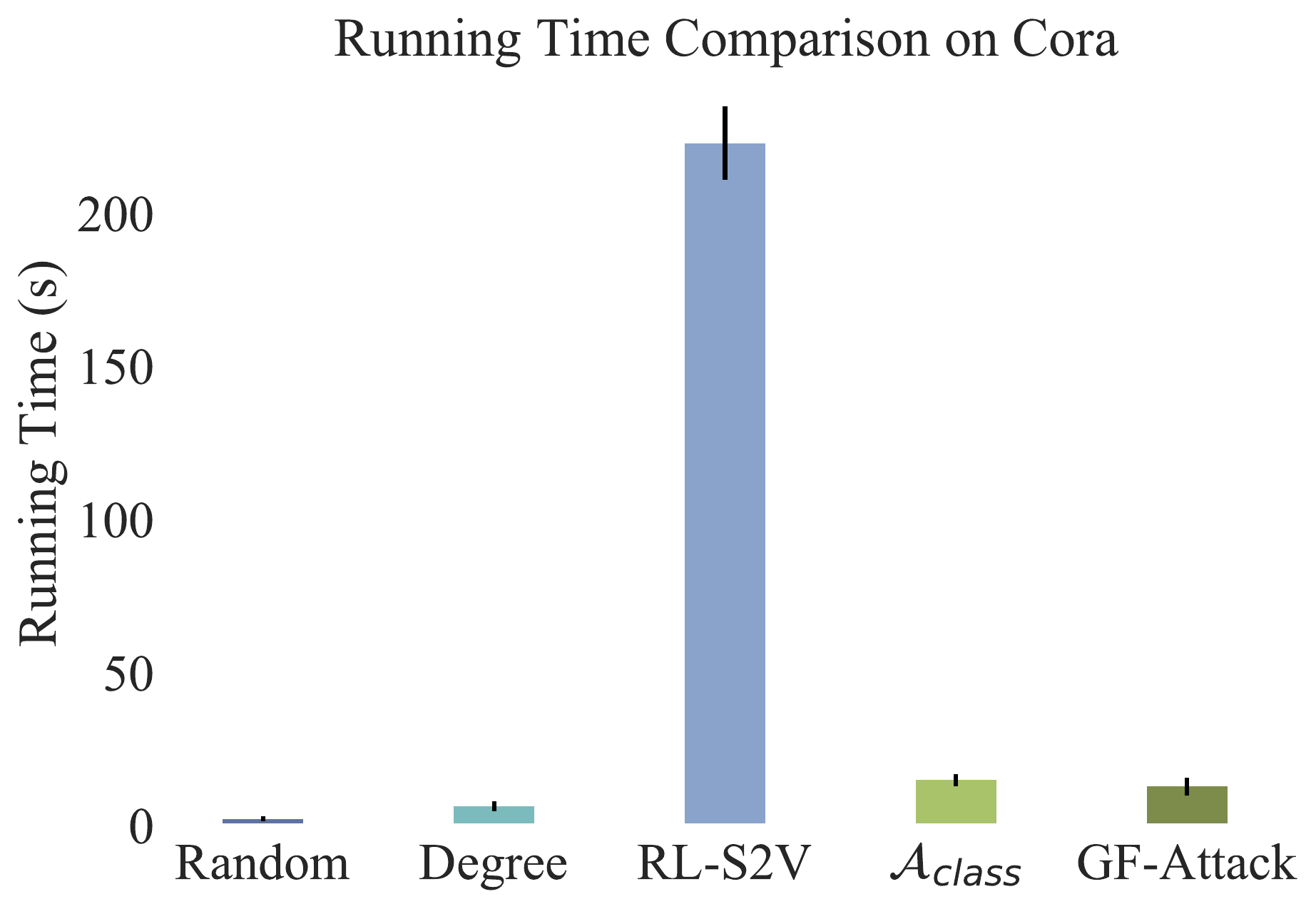}
\includegraphics[width=0.33\textwidth]{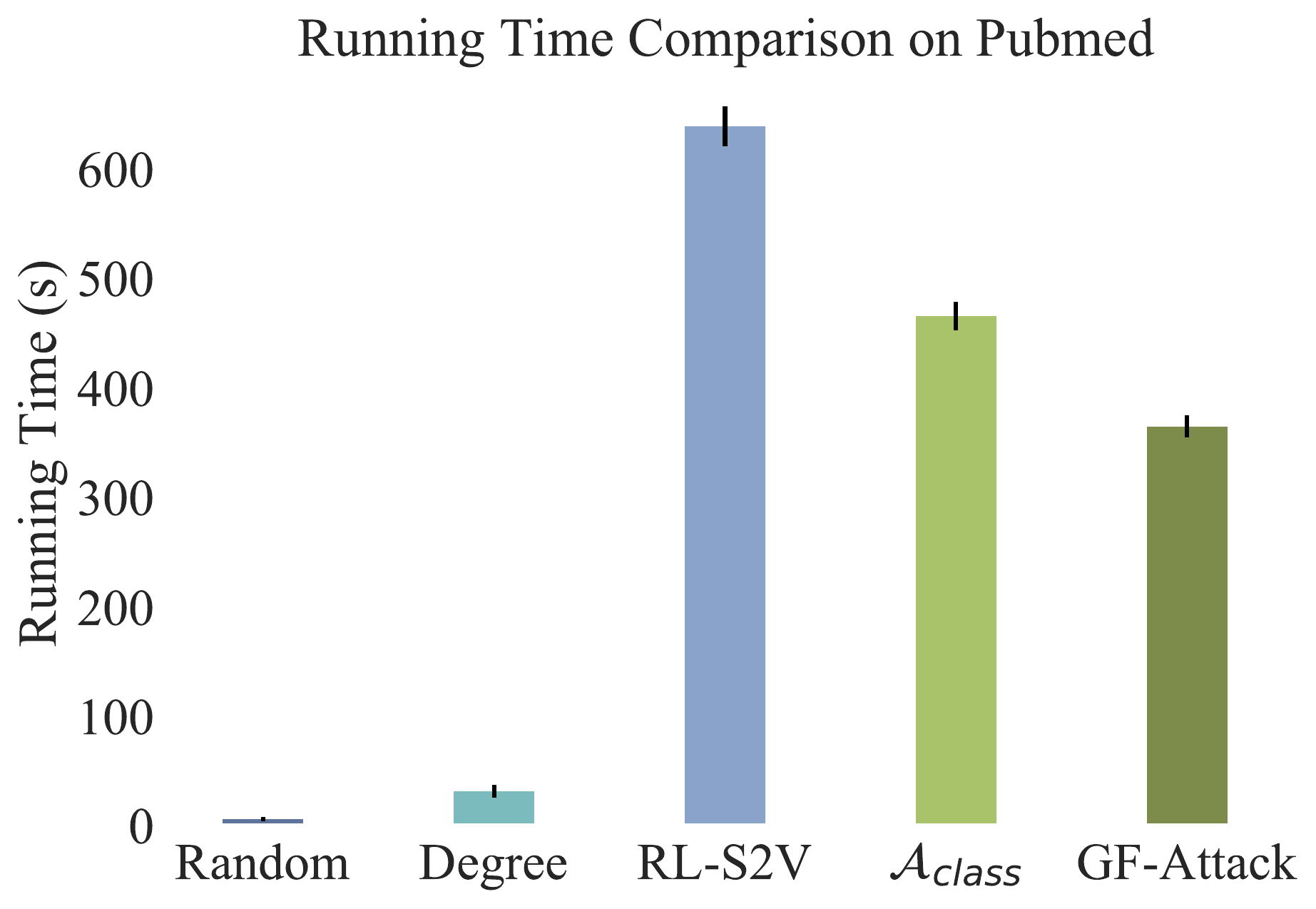}
% \vspace{4mm}
\caption{Running time ($s$) comparison overall baseline methods on all datasets. We report the $10$ times average running time of processing a single vertex for each model and the error bars are at the top of each bar.} \label{fig:timeCost}
\vspace{4mm}
\end{figure*}

\begin{figure*}[!t]
\centering
\subfigure [GCN]{\includegraphics[width=0.24\linewidth]{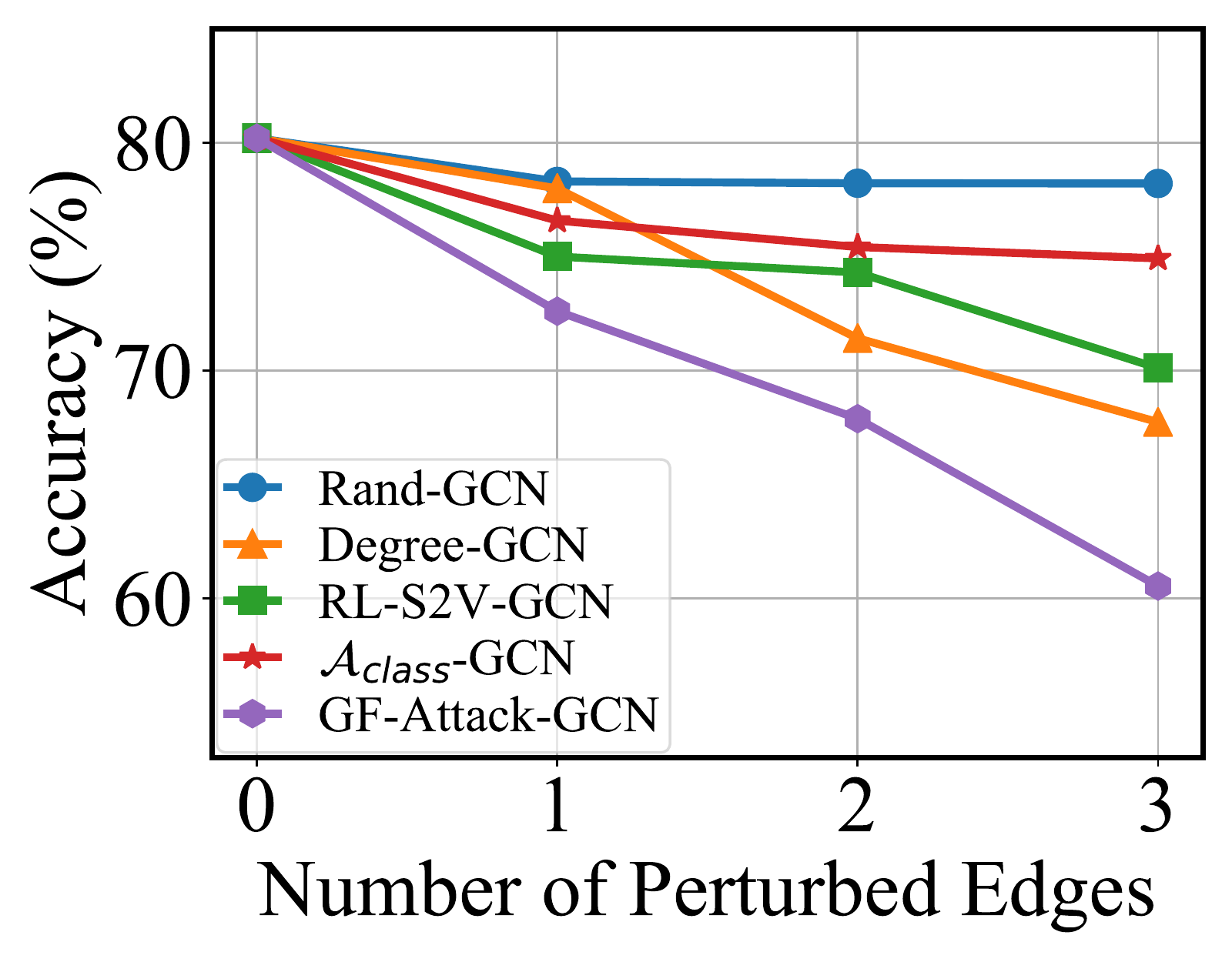}}
\subfigure [SGC]{\includegraphics[width=0.24\linewidth]{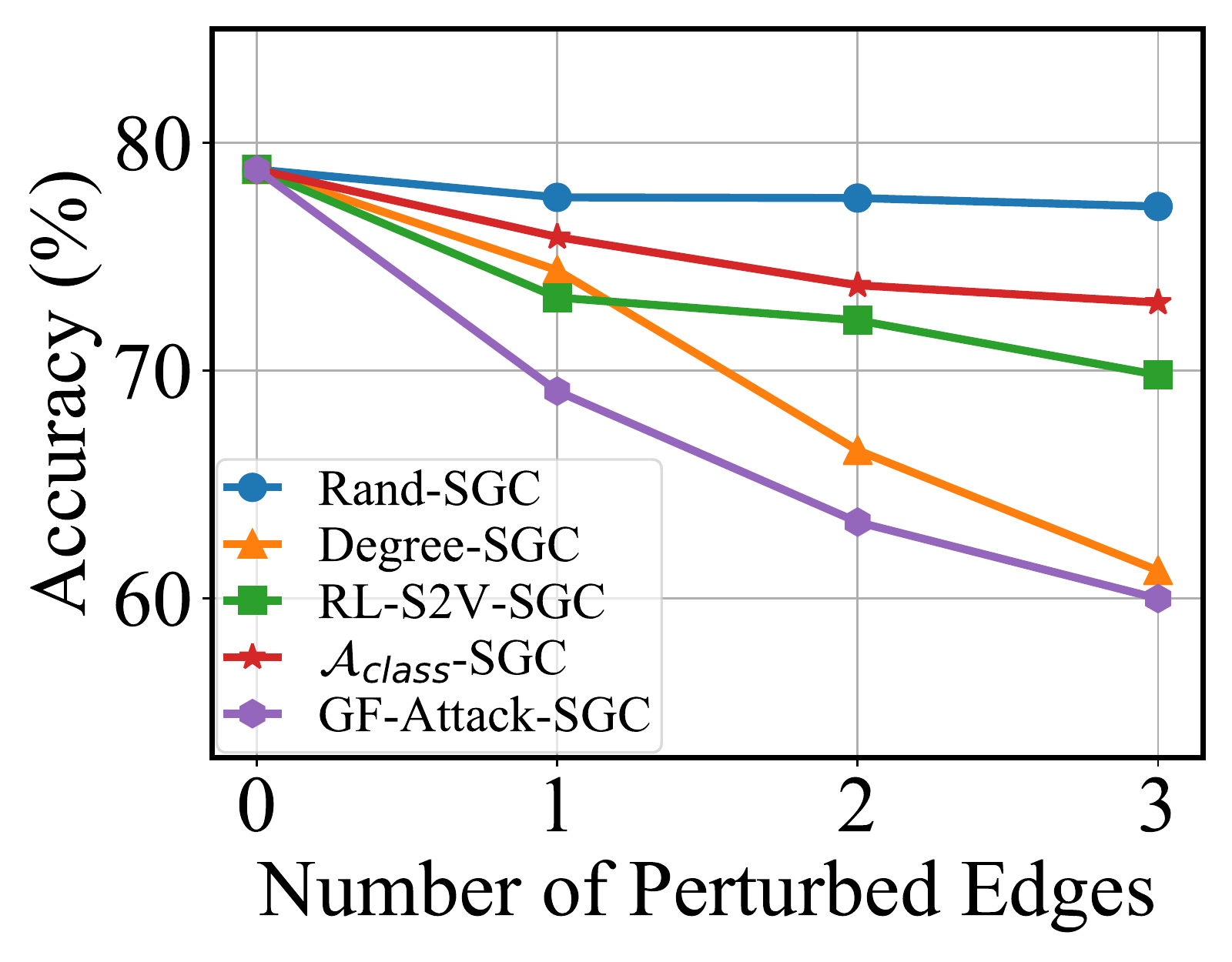}}
\subfigure [DeepWalk]{\includegraphics[width=0.24\linewidth]{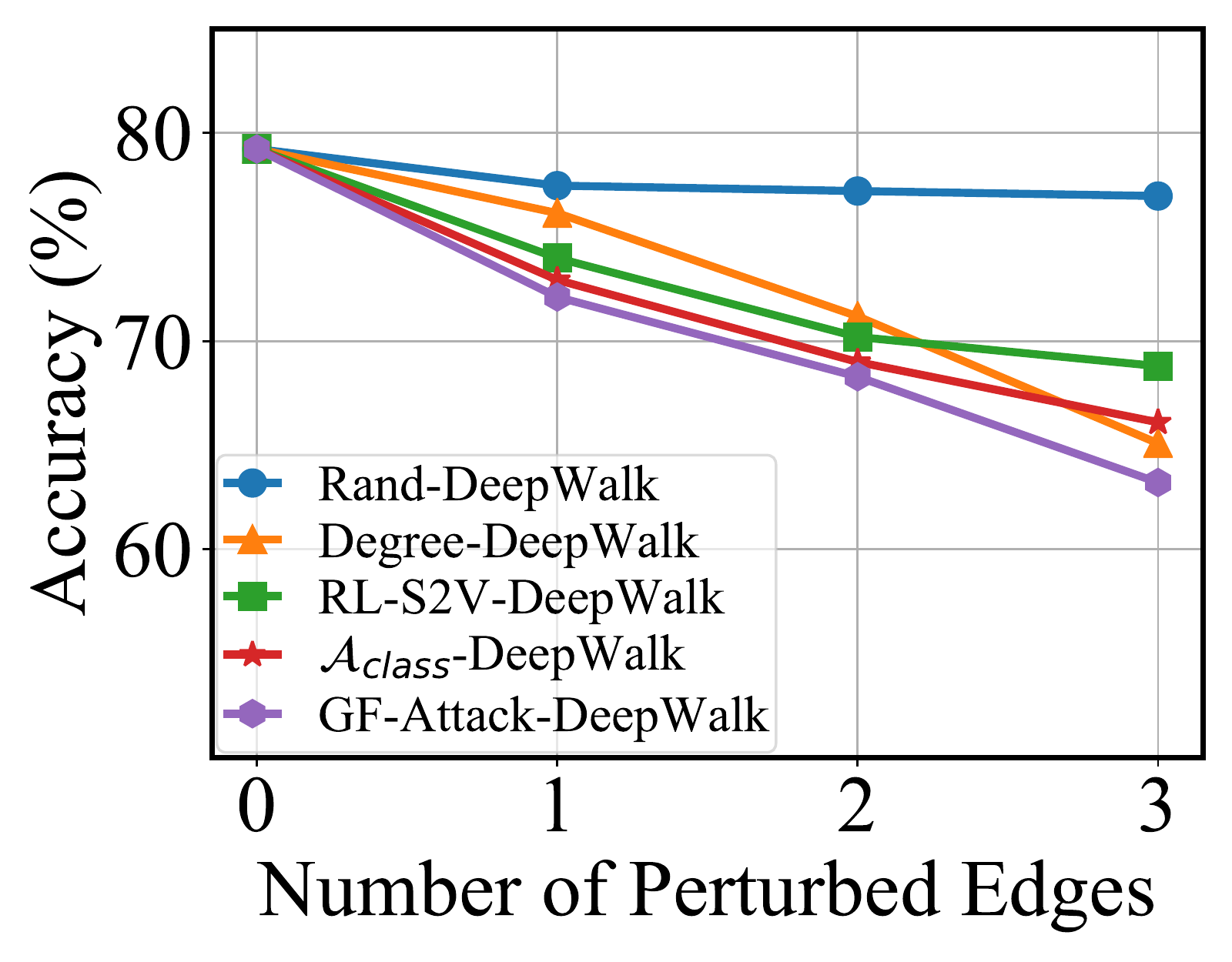}}
\subfigure [LINE]{\includegraphics[width=0.24\linewidth]{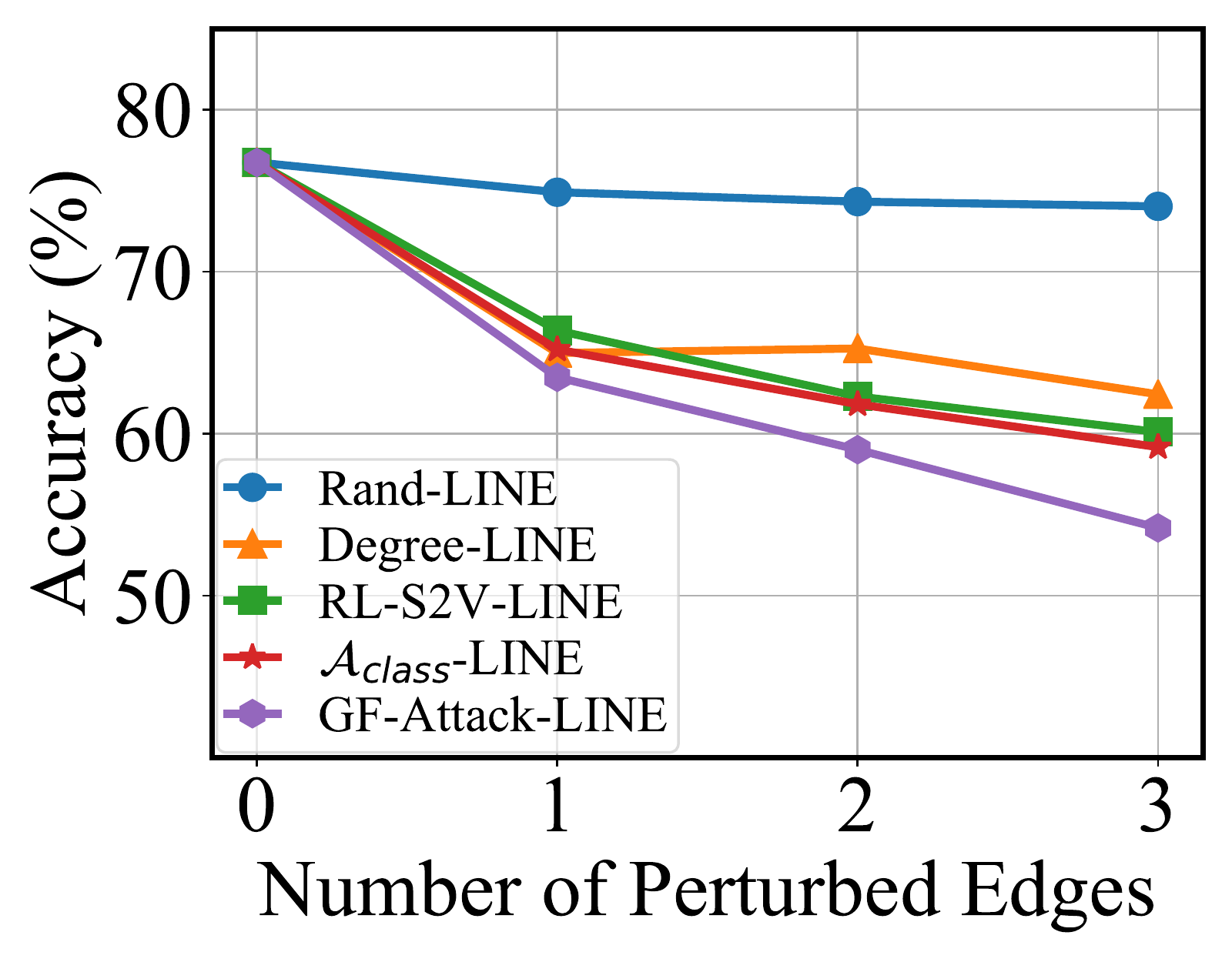}}
\caption{Multiple-edge attack results on Cora under RBA setting. Lower is better.}
\label{fig:multi-edge perturbation}
\end{figure*}

\begin{table}[!t]
\centering
\caption{Summary of the change in classification accuracy (in percent) compared to the clean/original graph. Single edge perturbation under the RBA and evasion settings. Lower is better. We use \textit{Cheby} for ChebyNet and \textit{DW} for DeepWalk here to save space. \label{tab:results-evasion}}
\resizebox{\columnwidth}{!}{%
\begin{tabular}{ l c c c c c c}
\toprule
    Dataset & \multicolumn{3}{c}{Cora} & \multicolumn{3}{c}{Citeseer}\\
\cmidrule(lr){2-4}\cmidrule(l){5-7}
Models & GCN & SGC & Cheby & GCN & SGC & Cheby\\
% \hline
(unattacked) &   80.20 & 78.82  & 80.33	& 72.50 & 69.68 & 70.96\\
\hline
\textit{Random} & -1.22 & -1.90  & -1.05 & -1.73 & -1.86  & -1.80\\
% \hline
\revision{\textit{Degree}} & -2.21 & -4.59  & -3.54 & -2.71 & -2.91  & -1.77\\
% \hline
\textit{RL-S2V} & -3.25 & -3.74  & -4.18 & -2.30 & -3.80  & -2.78\\
% \hline
\revision{\textit{$\mathcal{A}_{class}$}} & -2.83 & -4.03  & -3.54 & -1.92 & -3.78  & -3.67\\
%\hline
\midrule
\textit{GF-Attack} &  \textbf{-4.76}  & \textbf{-5.23} & \textbf{-4.54} & \textbf{-4.14} & \textbf{-5.33} & \textbf{-4.96}\\
\bottomrule
\end{tabular}
}
\end{table}

\subsection{Attack Performance Evaluation}
In this section, we combine the original results from the AAAI version~\cite{chang2020restricted} and evaluate the overall attack performance of different attackers. \revision{Note that the sampling-based GEMs can only be attacked under the poisoning setting~\cite{icml2019adversarial}, since sampling-based GEMs rely on training to generate new embeddings for perturbed graphs. Thus here we choose to perform the attack on all victim models under this setting and damage GCNs under the evasion setting alone in Section~\ref{sec.evasion}. Meanwhile, we choose to use Eq.~\eqref{equ.GF-Attack-sym} as the attack loss for all victim models here. In contrast to our AAAI version~\cite{chang2020restricted} which uses different attack losses for different types of GEMs, we take a step further to
better demonstrate the effectiveness of \textit{GF-Attack} under a more black-box setting, since this new setting removes the assumption of what type the victim model is.}

\textbf{Attack on GCNs.}
Table \ref{tab:results single edge} summarizes the attack results of different attackers on GCNs. Our \textit{GF-Attack} outperforms other attackers on all datasets and all models, even on the more complex parameterized-filter model ChebyNet. Moreover, \textit{GF-Attack} performs quite well on 2 layers GCN with nonlinear activation. This implies the generalization ability of \textit{GF-Attack} on GCNs as discussed in Section~\ref{sec.GF-GCN}. 

\textbf{Attack on Sampling-based GEMs.}
Table~\ref{tab:results single edge} also summarizes the results of different attackers on the sampling-based GEMs. As expected, \textit{GF-Attack} achieves the best performance nearly on all victim models. It validates the effectiveness of \textit{GF-Attack} on attacking sampling-based GEMs. 

Another interesting observation is that the attack performance on LINE is much better than that on DeepWalk. This result may due to the deterministic structure of LINE, while the random sampling procedure in DeepWalk may help raise its resistance to adversarial attacks. Moreover, \textit{GF-Attack} on all graph filters successfully drop the classification accuracy on both 
GCNs and sampling-based GEMs, which again indicates \revision{the transferability of the adversary examples generated by our general framework in practice.}

\subsection{Evaluation of Multi-layer GCNs and Multi-window-size Sampling-based GEMs.}
To further investigate the transferability of our framework, we conduct attacks towards different multi-layer GCNs and multi-window-size sampling-based GEMs \wrt the order of graph filter under our \textit{GF-Attack} framework supplementary to the original AAAI version~\cite{chang2020restricted}. 

% \textbf{Attack on Graph Convolutional Networks.}
Figure~\ref{fig:layer vs order GCN}, Figure~\ref{fig:layer vs order SGC} and Figure~\ref{fig:layer vs order DW} present the attack results on $2$, $3$, $4$ and $5$ layers GCN and SGC, and DeepWalk with window-size $1$ (LINE), $2$, $3$, $4$ and $5$ on Citeseer. The number followed by \textit{GF-Attack} indicates the graph filter order $K$ used in the attack loss. 
From Figure~\ref{fig:layer vs order GCN} to Figure~\ref{fig:layer vs order DW}, we can have some interesting observations: 
\begin{itemize}
    \item All the adversarial losses with different orders $K$ can perform successful attacks on all models, which again indicates the effectiveness of \textit{GF-Attack}.
    
    \item Particularly, \textit{GF-Attack-5} achieves the best-attack performance in most cases. It implies that the higher-order filter contains more fruitful information and has positive effects on the attacks targeting simpler models. This finding is consistent with the Theorem~\ref{thm:GCN-no-order-K}.
    % which implies the adversarial attack loss function~\eqref{equ.GF-Attack-sym} with larger $K$ is the lower bound of the losses with smaller $K$s. 
    
    \item The attack performance on SGC seems better than GCN under most of the settings. We conjecture that the non-linearity between layers in GCN can enhance the robustness of GCN.
    
    \item The performance of the adversarial attack on DeepWalk is better when the window-size grows for window-size ranging from 2 to 5. This is consistent with the mechanism of DeepWalk since when the window-size is larger, vertices from the further neighborhood of the target vertex will participate in learning embeddings. 
    % Since \textit{GF-Attack} selects the edges as candidates within the 1-hop neighborhood of the target vertex, the robustness of DeepWalk could be enhanced with a larger window-size. 
    % As for the results for 1-window DeepWalk, though it can be viewed as LINE from the matrix factorization perspective, it is still different in implementation since we use different public codes for LINE and DeepWalk, which results in the different trend.
\end{itemize}

\subsection{Evaluation under Multi-edge Perturbation Setting}
%the comparison between our unified models(one model for both DeepWalk and GCN) and others.
In this section, we evaluate the performance of attackers with multi-edge perturbation, i.e. $\beta \geq 1$, on all models supplementary to the original AAAI version~\cite{chang2020restricted}. 
The results of multi-edge perturbations on the Cora dataset under the RBA setting are reported in Figure~\ref{fig:multi-edge perturbation}.
Clearly, with the increase of the number of perturbed edges, the attack performance gets better for each attacker. \textit{GF-Attack} outperforms all the other baselines in all cases. It validates that \textit{GF-Attack} can still perform well when $\beta$ becomes larger. 

\subsection{Evaluation under Evasion Setting}\label{sec.evasion}
Since we mainly conduct analysis under the evasion setting in this work, we further investigate the performance of our framework under this setting with one-edge perturbation to demonstrate the ability of \textit{GF-Attack}. As shown in Table~\ref{tab:results-evasion}, we observe that the performance of all attack methods is degraded under the evasion attack setting, which implies that the GEMs could be misled by the adversarial examples during training under the poisoning setting. Further, \textit{GF-Attack} still consistently outperforms all baselines, though it is not specifically designed for the poisoning attacks.

% \begin{figure*}[htbp]
% \centering
% \subfigure [small][GCN]{\includegraphics[width=0.24\linewidth]{smoothcompare_GCN.pdf}}
% \subfigure [small][SGC]{\includegraphics[width=0.24\linewidth]{smoothcompare_SGC.pdf}}
% \subfigure [small][DeepWalk]{\includegraphics[width=0.24\linewidth]{smoothcompare_DeepWalk.pdf}}
% \subfigure [small][LINE]{\includegraphics[width=0.24\linewidth]{smoothcompare_LINE.pdf}}
% \caption{Multiple-edge attack results on Cora under RBA setting. Lower is better.}
% \label{fig:multi-edge perturbation}
% \end{figure*}

\subsection{Computational Efficiency Analysis}
%the comparison between our unified models(one model for both DeepWalk and GCN) and others.
In this section, we empirically evaluate the computational efficiency of our \textit{GF-Attack}. 
% The running time ($s$) comparison of $10$ times average on all datasets is demonstrated in Figure~\ref{fig:timeCost}. 
\revision{A comparison of the average values of the running time for $10$ runs of our algorithm for all datasets is given in Figure~\ref{fig:timeCost}.}
While being less efficient than two native baselines (\textit{Random} and \textit{Degree}), our \textit{GF-Attack} is much faster than the novel baselines \textit{RL-S2V} and \textit{$\mathcal{A}_{class}$}. Combining the performance in Table~\ref{tab:results single edge}, it reads that \textit{GF-Attack} is not only effective in performance but also efficient computationally.

\section{Conclusion}\label{sec.conclusion}
In this paper, we consider the adversarial attack on different kinds of GEMs under the restricted black-box attack scenario. From the view of graph signal processing, we try to formulate the procedure of graph embedding methods as a general graph signal processing with the corresponding graph filters. Then we construct a restricted adversarial attack framework which aims to attack the graph filter only by the adjacency matrix and the feature matrix. Thereby, a general optimization problem is constructed by measuring the embedding quality and an effective algorithm is derived accordingly to solve it. Experiments show the vulnerability of different kinds of novel GEMs to our general attack framework.

%\appendix[Proof of the Zonklar Eq.s]
% or


% \appendices
% \section{Proof of the First Zonklar Eq.}

% % you can choose not to have a title for an appendix
% % if you want by leaving the argument blank
% \section{}

\ifCLASSOPTIONcompsoc
  \section*{Acknowledgments}\label{sec.acknowledgment}
\else
  \section*{Acknowledgment}\label{sec.acknowledgment}
\fi

This work is supported in part by the National Key Research and Development Program of China (No. 2020AAA0106300, No. 2018AAA0102004),
the National Natural Science Foundation of China (No. 62102222, No.62006137, No. U1936219, No. 62141607), 
Tencent AI Lab Rhino-Bird Visiting Scholars Program (VS2022TEG001),
and the 2020 Tencent AI Lab Rhino-Bird Elite Training Program. We would like to thank Daniel Zügner from the Technical University of Munich for his valuable suggestions and discussions.

\ifCLASSOPTIONcaptionsoff
  \newpage
\fi



% \bibliographystyle{IEEEtran}
% % argument is your BibTeX string definitions and bibliography database(s)
% \bibliography{GF_Attack}

\printbibliography %added

\end{document}